\documentclass[format=acmsmall, review=false]{acmart}
\usepackage{acm}
\usepackage{booktabs} 
\usepackage[ruled]{algorithm2e} 

\SetAlFnt{\small}
\SetAlCapFnt{\small}
\SetAlCapNameFnt{\small}
\SetAlCapHSkip{0pt}
\IncMargin{-\parindent}

\usepackage{fvextra} 
\usepackage{algorithmic}
\usepackage{newfloat}
\usepackage{listings}
\DeclareCaptionStyle{ruled}{labelfont=normalfont,labelsep=colon,strut=off} 
\floatstyle{ruled}
\newfloat{listing}{tb}{lst}{}
\floatname{listing}{Listing}

\usepackage{tcolorbox}
\tcbuselibrary{listingsutf8}
\usepackage{capt-of}
\usepackage{subcaption} 
\usepackage{graphicx}

\newtcolorbox[auto counter]{mybox}[2][]{
  float,
  floatplacement=tbp,
  title={#2},
  colback=gray!10,
  colframe=gray!50,
  coltitle=black,
  #1
}

\setcitestyle{authoryear}

\title[An Interpretable Automated Mechanism Design Framework with Large Language Models]{An Interpretable Automated Mechanism Design Framework with Large Language Models}

\author[JL]{Jiayuan Liu}
\orcid{}
\affiliation{%
  \institution{Carnegie Mellon University}
}
\email{jiayuan4@andrew.cmu.edu}

\author[MG]{Mingyu Guo}
\orcid{}
\affiliation{%
  \institution{University of Adelaide}
}
\email{mingyu.guo@adelaide.edu.au}

\author[VC]{Vincent Conitzer}
\orcid{}
\affiliation{%
  \institution{Carnegie Mellon University}
}
\email{conitzer@cs.cmu.edu}

\begin{abstract}
Mechanism design has long been a cornerstone of economic theory, with traditional approaches relying on mathematical derivations.  More recently, {\em automated} mechanism design approaches have started to be used, most recently including {\em differentiable economics} methods that leverage neural networks to design payments and allocations (e.g., RegretNet and GemNet). While both analytical and automated methods have advanced the field, they each face significant weaknesses: mathematical derivations are not automated and often struggle to scale to complex problems, while automated and especially neural-network-based approaches suffer from limited interpretability. To address these challenges, we introduce a novel framework that reformulates mechanism design as a code generation task. Using large language models (LLMs), we generate heuristic mechanisms described in code and evolve them to optimize over some evaluation metrics while ensuring key design criteria (e.g., feasibility, strategy-proofness) through a problem-specific fixing process. This fixing process ensures that any mechanism violating the design criteria is adjusted to satisfy them, albeit with some trade-offs in performance metrics. These trade-offs are factored in during the LLM-based evolution process. The code generation capabilities of LLMs enable the discovery of novel and interpretable solutions, bridging the symbolic logic of mechanism design and the generative power of modern AI. Through rigorous experimentation, we demonstrate that LLM-generated mechanisms achieve competitive performance while offering greater interpretability compared to previous approaches. Notably, our framework can rediscover existing manually designed mechanisms (e.g., Myerson's mechanism in special cases and virtual valuation, reserve price auction, existing redistribution mechanism with theoretical guarantee) and provide insights into neural-network based solutions (i.e., RegretNet) through Programming-by-Example. These results highlight the potential of LLMs to not only automate but also enhance the transparency and scalability of mechanism design, thereby helping to ensure that we can safely deploy the resulting mechanisms in our societies.
\end{abstract}

\begin{document}

\begin{titlepage}

\maketitle

\vspace{1cm}
\setcounter{tocdepth}{1} 
\tableofcontents

\end{titlepage}

\section{Introduction}

{\em Mechanism design} is the theory at the heart of market design.  Its appeal is due to its rigorous mathematical foundation in game theory, foundational results such as the revelation principle, which allow us to clearly conceptualize what we seek to optimize, and key findings within the framework. These include positive results such as the VCG mechanisms~\cite{vickrey1961counterspeculation,clarke1971multipart,groves1973incentives} and Myerson's optimal auction~\cite{Myerson81:Optimal}, as well as negative results such as the Myerson-Satterthwaite impossibility~\cite{myerson1983efficient}.

Unfortunately, one way in which mechanism design has fallen short is that many important problems appear intractable to the analytical approach.  Perhaps the best-known example of this is the difficulty of generalizing Myerson's optimal auction in several directions, including multi-item auctions~\cite{daskalakis2015multi}, correlated bidder valuations~\cite{papadimitriou2011optimal}, and dynamic settings~\cite{doepke2006dynamic,zhang2021automated}. These extensions often lead to mechanisms that are analytically intractable. As a result, researchers have been forced to explore alternative approaches, such as approximate mechanism design~\cite{procaccia2013approximate,hartline2013mechanism} and data-driven methods~\cite{bergemann2024data}, to address more complex settings where analytical solutions remain elusive. 

In 2002, {\em automated mechanism design} (AMD)~\cite{Conitzer02:Complexity} was introduced as a way for computer scientists to help address this challenge. The idea is that, since mechanism design problems are naturally framed as optimization problems, we can approach those problems computationally and thereby solve for the optimal mechanism, at least in special cases.  While this approach has seen some successes, in general it faces its own challenges, including both {\em (computational) tractability} -- without any structure, linear-programming formulations do not scale well -- and {\em interpretability} -- unless we are optimizing only within a well-understood parameterized class of mechanisms, the solution may end up being presented to us as a giant inscrutable table of numbers from which little intuition can be gained.

A new approach within AMD that promises to make headway on the tractability challenge is {\em differentiable mechanism design}~\cite{RochetNet,GEMNet,shen2018automated,rahme2021permutation,feng2018deep}, relying on neural networks.  Unfortunately, as neural networks are famous for their lack of interpretability, the other challenge remains.

In this paper, we aim to introduce a new automated mechanism design framework that promises to address the interpretability issue. Classically, how do we describe the mechanisms found using the analytical approach—VCG mechanisms, Myerson's auction, etc.? Not as neural networks or giant tables of numbers, but as formulas, as code. This suggests a role for large language models (LLMs): If we can use LLMs to generate an (ideally brief) specification of how the mechanism runs, then we can analyze and understand this specification far more naturally than any giant table of numbers. Rapid progress in LLM, along with advances in associated techniques such as neurosymbolic approaches that combine neural networks with symbolic reasoning and improvements in reasoning abilities, suggests that this approach is likely to grow in success over the years.

Our primary contribution is to build an interpretable automated mechanism design framework with LLMs. 
Our method has the following advantages compared to the existing work. 
\begin{enumerate}
    \item Our method is automated. For a specific setting, even if it is possible to manually design a good mechanism and analytically derive its expected or worst-case guarantees, in general the methods used to achieve this are highly ad hoc. Even slight variations in the problem setup may require an entirely new analytical technique. 
    In contrast, our techniques, like AMD techniques in general, automate the process, generating mechanisms without requiring manual intervention, which are more flexible and applicable to complex and dynamic environments where manual design is infeasible. 
    \item Our method results in interpretable specifications of mechanisms. Unlike neural network-based approaches, which often produce ``black-box'' solutions, our framework generates mechanisms in the form of human-readable code. This can improve our understanding of these mechanisms and thereby refine our intuitions in general.  Even if we are automatically designing mechanisms for the purpose of their direct deployment in practice (rather than for the purpose of better economic-theoretic understanding), their interpretability is valuable in that it may allow provable guarantees, for example, from the perspectives of fairness, safety, and other ethical concerns.  It may also result in more buy-in into the mechanism from other parties. 
    Interpretability also contributes to safety, as these human-readable mechanisms can be rigorously tested and validated prior to deployment. 
    \item Our method performs well, and, as we argued, its performance is likely to improve along with improvements in LLMs and related techniques.  
\end{enumerate}

The remainder of this paper is laid out as follows. 
Sec.~\ref{sec:relatedwork} reviews related work in automated mechanism design, differentiable economics, LLMs, and evolutionary algorithms. 
Sec.~\ref{sec:framework} introduces our main method AMD-LLM that frames mechanism design as a code generation task using LLMs, and describes its workflow along with its key components. In addition, this section extends the framework to NNA-AMD-LLM by incorporating neural networks for enhanced performance. 
Sec.~\ref{sec:rediscover} demonstrates our method's ability to rediscover the virtual valuation function in independent bidders auctions. 
Sec.~\ref{sec:redistribution} applies the framework to VCG redistribution mechanism design in multi-bidder multi-unit unit-demand auctions, as an application example to social welfare maximization problems. 
Sec.~\ref{sec:correlated} explores how our framework can be applied to revenue maximization problems in a correlated bidders auction setting, as an application example to revenue maximization problems. 
Sec.~\ref{sec:conclusion} concludes the paper and proposes some future research directions.

\section{Related Work}\label{sec:relatedwork}

In this section, we review the various lines of related research on which this paper is based.

\subsection{Automated Mechanism Design}
The analytical results of optimal auction design have been limited to simple mechanism design settings such as variations of single-bidder settings. Theoretical studies have not been able to explore the full intricacy of mechanism design; a new direction, automated mechanism design (AMD), was introduced to help in this exploration. 

Discretizing the players' type spaces allows for the AMD problem to be solved as a linear-programming problem (or, if a deterministic mechanism is desired, a mixed-integer linear program)~\cite{Conitzer02:Complexity,Conitzer04:Self,conitzer2003applications}.
But the cost of such generality is that it is hard to scale this approach.
Several more scalable approaches to AMD have been developed to address scalability issues. One such approach is parametric mechanism design, which restricts the search space to a specific class of parameterized mechanisms. This method has proven effective in various settings, including helping conjecture and proving analytical results. Examples include auction design~\cite{Sandholm15:Automated,sandholm2003automated} and the design of mechanisms that redistribute as much of the auction revenue back to bidders as possible~\cite{cavallo2008efficiency,guo2016competitive,guo2007worst,guo2008undominated,guo2024worst,guo2014better,guo2012worst,guo2011vcg,guo2010optimal,guo2008better} (for an overview of this approach, see~\cite{Guo10:Computationally}).  
On the other hand, in a sense this approach is only partly automated, in that specifying a good parameterized class of mechanisms still requires human insight. Consequently, this approach has been successful only in some domains.

\subsection{Differentiable Auction Design and Neurosymbolic AI} 
More recently, differentiable methods have been introduced that use neural networks to do auction design~\cite{RochetNet,GEMNet,shen2018automated,rahme2021permutation,feng2018deep}. The main challenge is to fulfill the expressiveness, strategy-proofness, and multi-bidder requirements. In terms of expressiveness, though the optimal auctions are represented by the neural network, the resulting mechanisms are hard to interpret from the weights of the neural network.

A recent direction for AI research, neurosymbolic approaches that combine neural networks with symbolic reasoning, is promising for solving the interpretability issue. These methods have shown success in learning interpretable program models, improving efficiency and generalization in various tasks, and enhancing AI system explainability~\cite{hitzler2022neuro,sarker2022neuro,sheth2023neurosymbolic}. By bridging neural computation with symbolic representation, neurosymbolic AI could potentially address interpretability challenges in mechanism design while maintaining computational efficiency.

\subsection{Large Language Models and Code Generation}
Large Language Models (LLMs) have revolutionized natural language processing, demonstrating remarkable capabilities in reasoning, problem-solving, and text generation~\cite{zhao2023survey}. Recent advancements have led to models with enhanced abilities in multi-step reasoning and logical deduction, such as GPT-o1 and DeepSeek-R1\footnote{GPT-o1: \url{https://platform.openai.com/docs/models\#o1}, DeepSeek-R1: \url{https://github.com/deepseek-ai/DeepSeek-R1}.}~\cite{xiang2025towards,guan2025rstar,liu2024deepseek,mercer2025brief}.

In the realm of code generation, LLMs have shown impressive capabilities~\cite{nijkamp2022codegen,fried2022incoder,chen2021evaluating}.
Code generation refers to the automatic creation of source code from high-level descriptions, specifications, or natural language prompts. LLMs have achieved notable performance in code generation tasks, such as solving problems from coding competitions~\cite{AlphaCode}, debugging~\cite{Program-Repair-Inferfix,chen2023teaching,liventsev2023fully}, and improving the given code~\cite{Code-edits}. One reason for LLMs' excellent performance is their training on vast datasets of both natural language and code, enabling them to understand context, generate syntactically correct code, and work across multiple programming languages.
These works provide the foundations for solving various tasks via LLM-powered code generation. 
Developers can already take advantage of existing LLM-based code generation tools such as GitHub Copilot\footnote{\url{https://github.com/features/copilot}} to assist with code explanation, debugging, and writing. 
Given the recent advancements in LLMs, it is natural to explore the application of LLM-based code generation techniques to mechanism design problems.

\subsection{Evolutionary Algorithms}
Evolutionary Algorithms (EAs) are a class of optimization techniques inspired by biological evolution. These algorithms leverage principles such as natural selection, mutation, and crossover to solve complex problems in computer science and engineering, iteratively improving a population of candidate solutions through score-based selection (often called ``fitness-based'' in the EA literature) and stochastic variation operators. 
EAs excel in navigating high-dimensional, multi-modal, and non-differentiable search spaces where traditional gradient-based methods often get stuck~\cite{yu2010introduction,slowik2020evolutionary}. 

Recent works combine EAs with LLMs, introducing EA-based prompt engineering, particularly for mathematical and algorithmic discovery.
FunSearch~\cite{FunSearch} iteratively generates improved solutions in the form of computer ``function'' code given previous good solutions using LLMs, verifies the generated solutions with an automated ``evaluator'', then maintains a pool of good solutions with the island-based evolutionary method to encourage exploration and avoid getting trapped at sub-optimal solutions.
Evolution of Heuristics (EoH)~\cite {liu2024evolution} proposes a similar heuristic evolution framework as FunSearch.  EoH explicitly specifies the evolution strategies and incorporates one of these strategies into the prompt in each iteration. The LLM can then understand this strategy prompt and evolve accordingly. 
Self-Taught Optimizer (STOP)~\cite{zelikman2023self} investigates sequential heuristic optimization techniques with other optimization techniques in addition to evolutionary algorithms. The techniques discussed in STOP include multi-armed bandits, tree search, and simulated annealing-based search. 
There are also techniques that utilize EAs to do automatic prompt improvement. An example is Promptbreeder~\cite{fernando2023promptbreeder}, which automatically evolves and optimizes the prompts for the LLM. 
It iteratively improves the task prompts and the mutation prompts, guiding their evolution.

\subsection{Programming-by-Example}
While LLMs are known for their expertise in deductive reasoning tasks (following the input instructions and reasoning as told), they also have strong capabilities in inductive reasoning tasks, where they infer underlying patterns from input-output pairs or cause-effect relationships~\cite{cheng2024inductive}. This capability makes them particularly well-suited for Programming-by-Example (PBE), a technique where a system automatically generates programs or algorithms from example input-output pairs. Instead of manually writing code, users provide examples of the desired behavior, and the system infers the rules to replicate and generalize that behavior for new inputs~\cite{gulwani2017program,halbert1984programming}.
PBE has been successfully applied in various domains, such as automating repetitive tasks in spreadsheets~\cite{wu2023programming}, and synthesizing data transformations~\cite{jin2017foofah,feser2015synthesizing}. This method significantly reduces the need for manual coding. 
Recent advancements in LLMs have expanded PBE's scope, enabling systems to handle more complex tasks with fewer examples. By leveraging their training data and reasoning capabilities, LLMs can infer patterns and generate robust, generalizable programs~\cite{cheng2024inductive}.

\section{Framing Mechanism Design as an Automated Code Generation Workflow}\label{sec:framework}

In the following sections, we denote bidder $i$'s bid as $b_i$, the bids of all other bidders as $b_{-i}$, and the bids of all bidders collectively as $\vec b$. Similarly, we define $p_i, p_{-i}$, and $\vec p$ for payments. 

\subsection{The Automated Mechanism Design with LLM (AMD-LLM) Framework}\label{sec:amd-llm}

\subsubsection{Key components of the framework}
We automate mechanism design using an LLM code generation workflow that requires three components to be specified:
\begin{enumerate}
    \item Identify the {\em key} part of the mechanism design problem and leave it {\em blank} as the {\em heuristic function} that will be evolved by LLMs. 
    Specifically, we run our experiments in the FunSearch environment~\cite{FunSearch}. Example~\ref{example:auction-spec} below describes a specific auction setting -- Myerson's optimal auction setting~\cite{Myerson81:Optimal}, but with bid correlation. In the context of this example, the key part of the mechanism design problem is the allocation function, which is expressed as a Python function that maps the bids to the allocation result. The initial {\em blank} version simply does not allocate, or we could use any other trivial initialiazation such as
    always allocating to bidder $1$.
    This allocation function will go through the LLM-based evolution process. According to Myerson's characterization, given
    an allocation function, there is only one way to charge the payments without violating either strategy-proofness or individual
    rationality, so the payment function does not need evolution in itself.
    That is, the code that derives
    the correct payments given the current allocation is manually provided as code boilerplate and does not change.
    \item Specify a problem-specific {\em fixing rule} or {\em fixing process} that ensures that the generated heuristic meets all desired mechanism design criteria --  e.g., feasibility, strategy-proofness, etc. -- after the fix is applied. See Sec.~\ref{sec:fix} for a formal definition. In Example~\ref{example:auction-spec}'s context, the auction is strategy-proof if and only if the allocation function is {\em monotone} (i.e., the winner must
    still win if she raises her bid while the other bids stay the same), which is, once again, based on Myerson's characterization.
    Our fixing rule converts any allocation function into a monotone version by raising the critical price faced by the winner to the minimum threshold above which the winner always wins.
    This fixing step is also via manually written code and it is part of the code boilerplate.
    \item Specify an {\em evaluation function} that evaluates the performance of the generated heuristic. 
    In the context of Example~\ref{example:auction-spec}, we simply apply Monte Carlo simulation to evaluate the expected revenue. First, we fix the evolved allocation heuristic function to ensure monotonicity, and then we derive the corresponding payments.
\end{enumerate}

\subsubsection{Fixing process}\label{sec:fix}
Although LLMs have strong code generation capabilities, the heuristic functions they generate may not always satisfy all mechanism design criteria, particularly for complex problems requiring lengthy, detailed prompts. If the generated mechanism fails to meet even one of the ten requirements, it becomes ineffective.
We introduce the fixing process to ensure these important criteria, at the cost of some performance loss in terms of the evaluation metric. 
The fixing process should be universally applicable to any mechanism in the design space. 
For example, in the context of Example~\ref{example:auction-spec}, we can always raise the critical
price faced by the winner to the point that any bid above it ensures that the winner still wins.
This fixing process can fix any allocation heuristic to reach monotonicity, but in some situations, we may need to raise
the critical price significantly, which results in lower expected revenue due to high probability of no allocation (no agents can afford the item).

\begin{definition}[Fixing process]
    A \textit{fixing process} (or \textit{fixing rule} or \textit{fix}) is a problem-specific post-processing procedure applied to any mechanism within a given design space that fails to satisfy one or more specified design criteria. 
    The fixing process transforms such a mechanism into one that meets all required criteria. 
    This procedure is tailored to the specific problem and design constraints, and it may involve trade-offs, such as performance loss, to achieve the desired criteria.
\end{definition}

For many mechanism design problems, a na\"ive fix usually exists. For instance, if the generated mechanism fails to meet the criteria, then the fix can simply change it to always produce a default outcome (e.g., no allocation in auction design), resulting in a zero performance score. 
More delicate fixes can be found by investigating the properties of the specific problem, as
shown in Example~\ref{example:auction-spec} for optimal single-item auction design.

We can similarly fix several other auction design problems. Another example involving fixing
the VCG redistribution function is given in Sec.~\ref{sec:redistribution}.

\subsubsection{The workflow}\label{sec:workflow-no-nn}
In each iteration, the evolutionary framework selects one or more existing heuristics from the program database (containing all saved heuristics so far), combines them together using a prompt (i.e., the task description and optionally an evolutionary strategy along the line of ``mutation'', ``crossover'', etc.), and then inputs the combined prompt to the LLM. Next, the LLM outputs a new heuristic. The heuristic will be fixed through the problem-specific fixing process, and evaluated using the evaluation function. The returned value of the evaluation function is called the fitness (or score) in the evolutionary computation literature; it indicates how well the generated heuristic function meets the desired objective, guiding selection and improvement during the evolutionary process. 
This process continues repeatedly, as shown in Fig.~\ref{fig:workflow} (following the ``without NN'' branch).

\begin{figure}[htbp]
    \centering
    \includegraphics[width=\textwidth]{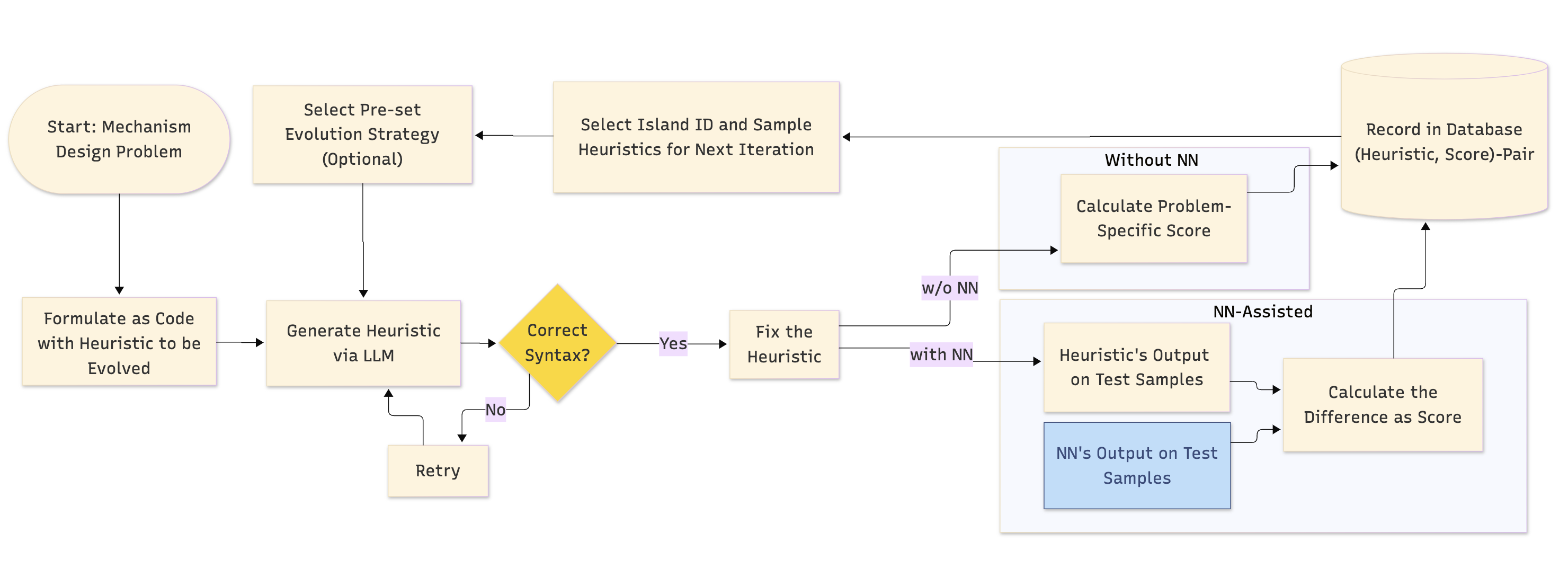}
    \caption{The workflow for frameworks AMD-LLM (following the ``without NN'' branch) and NNA-AMD-LLM (following the ``with NN'' branch). }
    \label{fig:workflow}
\end{figure}

\subsubsection{The Prompts}
There are two kinds of prompts for LLMs: system prompts and user prompts.

\vspace{.1in}
\noindent A
\textit{system prompt} is an instruction that guides the LLM's behavior, tone, and constraints. In our framework, it includes information such as ``you are a mechanism design expert,'' a description of the task, and the required format of the heuristic output. This prompt remains unchanged throughout the workflow.

\vspace{.1in}
\noindent A
\textit{user prompt} is the input or query given by the user to generate a response. In our framework, it contains iteration-specific content, including sampled past heuristics and optionally a sampled pre-set evolution strategy (i.e., ``mutate'', ``crossover'', etc., details in Sec.~\ref{sec:preset-strategies}). The user prompt varies for each iteration.

\subsubsection{The Pre-set Evolution Strategies}\label{sec:preset-strategies}

An evolutionary algorithm normally includes the following key genetic operations.
\begin{itemize}
    \item Crossover: combines material from parents to produce offspring, facilitating exploration of the search space by generating new solution combinations.
    \item Mutation: introduces random modifications to individuals, maintaining population diversity, and preventing premature convergence to suboptimal solutions.
    \item Selection and elitism: selects the desirable offspring to retain in the population, and meanwhile, ensures that high-quality solutions are not lost during evolution.
\end{itemize}

We can pre-set the evolution strategies through explicit prompts and ask the LLM to evolve accordingly. 
\citet{liu2024evolution} used this idea to generate combinatorial optimization heuristics, where they designed five strategies, explicitly asking the LLM to create entirely new algorithms, to draw inspiration from existing algorithms, and to adapt or refine existing algorithms by altering their structure, parameters, or components to improve performance or generalization. These prompts well depict the key genetic operations in evolutionary algorithms and can be generally applied. 
We alter these general strategy prompts to more ad-hoc ones to better fit each of our tested settings (see the examples in later sections). 

It should be noted that explicitly specifying the evolution strategy in the prompt is optional. 
If we want to explicitly pre-set an evolution strategy, then in each iteration, the evolutionary framework will select one of the strategies and combine it into the user prompt and input it into the LLM. 
If we do not explicitly prompt the model about evolutionary strategies (as in the original FunSearch paper), the LLM will autonomously determine how to generate the new heuristic.

\subsubsection{A Detailed Example}
\begin{example}\label{example:auction-spec}
    Code specification~\ref{spec:example-one-item} demonstrates an example of our AMD-LLM workflow for solving Myerson's optimal single-item auction design problem~\cite{Myerson81:Optimal}, with the additional complexity of allowing bid correlation.\footnote{Optimal single-item auction under correlation is known to be computationally hard~\cite{papadimitriou2011optimal}.}
    In this mechanism design setting, a single item must be deterministically allocated to one of the bidders, or left unallocated. The bidders' valuations can be correlated. The mechanism designer aims to maximize the expected total revenue. The design criteria include feasibility (i.e., no over allocation), strategy-proofness, and ex post individual rationality. 
    \begin{enumerate}
        \item We first identify the key part of the mechanism and set it as the heuristic function that will evolve through FunSearch. The key part we choose in Code Specification~\ref{spec:example-one-item} is the allocation function; see lines 36-38. The input of the heuristic function consists of the bidders' bids (a vector of length $n$ where $n$ is the number of bidders). The output of the heuristic function is the allocation function, which is a vector of length $n+1$ with each element in $[0, 1]$; if the $i$-th ($1\leq i \leq n$) index of the output vector is the largest, then allocate the item to the $i$-th bidder; if the $(n+1)$-th index is the largest, then do not allocate to any bidder. 
        \item We need to design a fixing process that ensures all the design criteria. Given a heuristic function (which serves as the allocation function), we need to ensure its monotonicity, which
        is required for strategy-proofness based on Myerson's characterization. Recall that monotonicity
        refers to the property that the winner must still win if she increases her bid while the other
        bids stay the same. We note that under a strategy-proof and individually rational
        mechanism, a bidder essentially faces a take-it-or-leave-it offer, sometimes referred to as the critical price, which depends on the other bids.
        Let the winner be bidder $m$.
        One way to fix a non-monotone allocation is to 
        increase the critical price faced by the winner to the lowest price $p_m$
        that satisfies $\forall b_m'\ge p_m$, we have $\text{argmax}\ \text{heuristic}(b_m',b_{-m})=m$.
        The revenue is $p_m$ if bidder $m$ can afford to pay this price, or $0$ otherwise.
        (All other bidders pay nothing as they don't win.)
        The above fixing process guarantees monotonicity. It is shown in code in lines 19-34 in Code Specification~\ref{spec:example-one-item}. 
        \item For the evaluation function, we sample a batch of bid vectors from the distribution, compute the average revenue of the fixed generated mechanism, and use it as the score for the FunSearch evolutionary update; see lines 3-12. 
    \end{enumerate}
    With this code specification, our AMD framework can call the LLM iteratively to evolve the heuristic and evaluate each generated heuristic according to the workflow shown in Sec.~\ref{sec:workflow-no-nn}. 
\end{example}

\begin{lstlisting}[basicstyle=\ttfamily\scriptsize,mathescape=true,caption={An Example Code Specification for Multi-Bidder Single-Item Auction Design},label={spec:example-one-item}]
import torch
import funsearch
@funsearch.run    # Evaluation function to be executed by FunSearch
def main(num_bidders) -> float:
  batch_size = 10000   
  batched_bids = get_bidders_values(batch_size, num_bidders)
  payment_vec = []
  for bids in batched_bids:
    payment = solve(bids, num_bidders)
    payment_vec.append(payment) 
  score = torch.mean(torch.tensor(payment_vec)).item()
  return score

def solve(bids, num_bidders):
  alloc = heuristic(bids)
  max_index = torch.argmax(alloc)
  if max_index == num_bidders:
    return 0
  # Following starts the fix that ensures monotonicity, and thus truthfulness
  epsilon = 0.00001  # Set this granularity accordingly 
  curr_bid = 1
  while curr_bid >= 0: 
    bids_cloned = bids.clone()
    bids_cloned[max_index] = curr_bid
    alloc_heu = heuristic(bids_cloned)
    alloc_heu = torch.tensor(alloc_heu)
    if torch.argmax(alloc_heu) == max_index:
      curr_bid -= epsilon
    else:
      break
  if curr_bid >= bids[max_index]:
    return 0
  else:
    return curr_bid

@funsearch.evolve    # The heuristic function to be evolved through FunSearch
def heuristic(bids): 
  return None  # Start with a naive heuristic
\end{lstlisting}

\subsection{Neural Network Assisted Automated Mechanism Design Framework with LLM (NNA-AMD-LLM)}\label{sec:nna-amd-llm}
This section integrates the existing differential economics methods --  neural network mechanism design, such as RegretNet~\cite{RochetNet} -- into our AMD-LLM framework. 
Neural networks are often capable of discovering highly effective mechanisms. However, they often lack interpretability, as the deliverables are black-box networks involving intricate arrangements of weights and biases. 

Building on the AMD-LLM framework introduced in Sec.~\ref{sec:amd-llm},  we incorporate an additional RegretNet, which is trained separately in advance to be used as input to our framework, serving as the ``goal'' function. Here, the fitness of an LLM-generated heuristic refers to its similarity to the trained RegretNet. 
The evolutionary process minimizes the difference between the outputs of the LLM-generated heuristic
and the trained RegretNet on the same batch of inputs, guiding the generated heuristics to approximate the trained RegretNet. The end result is a Python code segment that serves as an interpretable approximation of RegretNet's black-box solution.
This aligns with the Program-by-Example (PBE) technique. Here, the ``examples'' are the input-output mappings of RegretNet, and the evolutionary algorithm acts as the program synthesizer. By iteratively generating code candidates and comparing their outputs to RegretNet’s, the process distills the neural network’s implicit logic into explicit, human-readable Python code. 
Unlike traditional PBE, which often uses deductive methods, this approach leverages evolutionary search to navigate the vast space of possible programs. The result is a programmatic explanation of RegretNet, bypassing its opaque internal computations while preserving its functional essence.  

See Fig.~\ref{fig:workflow} (following the ``with NN'' branch) for the workflow of NNA-AMD-LLM.

\section{An Initial Application -- Rediscovering Virtual Valuations}\label{sec:rediscover}

We start with a simple setting to see if our framework is effective enough to rediscover an existing mechanism with moderate complexity. We choose Myerson's single-item auction setting, 
where the optimal auction has been derived by~\citet{Myerson81:Optimal}. There is a single item for sale and there are multiple bidders with independent private valuations $v_i$. The seller wants to design a deterministic, strategy-proof, and ex post individually rational auction that maximizes the expected revenue. She does not know the bidders' exact valuations, but she knows that $v_i$ follows a distribution with cumulative distribution function $F(v_i)$ and probability density function $f(v_i)$.\footnote{For simplicity, we assume the distributions
 are both independent and identical.} 

\citet{Myerson81:Optimal} defines the {\em virtual valuation} to be $r(v_i):=v_i-\frac{1-F(v_i)}{f(v_i)}$.
\citet{Myerson81:Optimal}  proves that the optimal auction in this setting is to allocate the item to the bidder with the highest virtual valuation, or not to allocate if none of the bidders' virtual valuations is non-negative; the expected revenue for any truthful mechanism equals the expected virtual valuation.\footnote{For simplicity, we assume distribution regularity.}

We use our AMD framework with LLM to solve this problem. In the code, we define two extra functions, pdf(v) and cdf(v), which the LLM can directly use in the generated heuristic function. 
The experiments in this section do not use pre-set evolution strategies. The bidders' values are sampled from the Beta(2, 5) distribution. 
The system prompt we use is shown in Prompt~\ref{prompt:rediscovery1}, in which we instruct the LLM to follow some rules, and to try using the pre-defined pdf(v), cdf(v) functions in the heuristics.

\begin{mybox}[label=prompt:rediscovery1]{Prompt \thetcbcounter. The System Prompt for Rediscovering Virtual Valuation}
\begin{Verbatim}[breaklines=true,fontsize=\scriptsize,breaksymbol={}]
You are a mechanism design expert. Your task is to design a heuristic function for a deterministic revenue-maximizing single-item auction with 1 item and 2 bidders where the two bidders' values are independent, i.e., the seller aims to sell the item to one of the two bidders, or not sell to anyone. The heuristic function's input v is a real number in range [0, 1] which is one bidder's bidding price, following a Beta distribution Beta(alpha, beta). Inputs alpha and beta are the parameters of the Beta distribution that v follows. The heuristic function's output must be a real number. Each bidder will call this heuristic function once and get his or her value. If the outputs of the heuristic function for both players are negative, then do not allocate to any bidder; if at least one of the bidders' heuristic function output values are non-negative, then allocate the item to the bidder with the largest heuristic function output value. The goal of the mechanism designer is to maximize the revenue while ensuring participants have incentives to bid truthfully (incentive compatibility) and are not worse off by participating (individual rationality). The input to the heuristic function is only one bidder's bid v, which is a real number in range [0, 1], as well as parameters alpha and beta, and there are no other args as input, you must not use other parameters without defining. The output must be a real number, can be positive or non-positive. You should remove redundant code (i.e., if a shorter version of code has equivalent function, please use the shorter one). You only need to design one new heuristic function. You must only output a standalone heuristic_v{version} function code, must not output anything other than the heuristic function's code. You should make the code concise and short, and also make the solution as general as possible (can be potentially applied to other distributions). You have the knowledge of the bids' distribution (which is Beta(alpha, beta) distribution) and can use these predefined functions in the heuristic: the probability distribution pdf(v), the cumulative distribution function cdf(v). These predefined functions are already defined in the previous code, so you just need to use pdf(v), cdf(v) in your code. You can also use basic mathematics calculations such as +, -, *, /, as well as other more complex combinations of calculations. In this case, you know that x follows a certain Beta(alpha, beta) distribution for x \in [0, 1]. You can use v, cdf(v), pdf(v) in your heuristic function. You must limit the code in 5 lines. You should also avoid using magic numbers. Please must not use randomization. You must use 2 spaces as indent for Python code. The only inputs are the bidder's bid v, alpha, and beta. No other inputs are allowed! The input v is a real number that represents the bidder's bidding price. You must not use alpha or beta in the heuristic function, instead, please use pdf(v), cdf(v) to describe distribution-related information. 
\end{Verbatim}
\end{mybox}

We experimented on different LLMs, including DeekSeek-V3\footnote{\url{https://github.com/deepseek-ai/DeepSeek-V3}}, DeepSeek-R1\footnote{\url{https://github.com/deepseek-ai/DeepSeek-R1}}, 
Llama-3.1-70B-Instruct\footnote{\url{https://huggingface.co/meta-llama/Llama-3.1-70B-Instruct}}, CodeGemma-7B\footnote{\url{https://huggingface.co/google/codegemma-7b}}, which are all open-weight models, as well as GPT-3.5-Turbo\footnote{\url{https://platform.openai.com/docs/models\#gpt-3-5-turbo}}, and GPT-4o\footnote{\url{https://platform.openai.com/docs/models\#gpt-4o}}.

We first present the results when we use GPT-3.5-Turbo in our AMD framework. Over time, the heuristic function can successfully evolve from ``return 0'' to outputing the virtual valuation formula after about 200 iterations. Afterwards, it still explores other possibilities and then it will find out that the virtual valuation works the best and will gradually converge on that. 
The following is an excerpt of the generated heuristics over time. We chose some of the heuristics generated along the way for purposes of illustration; they are presented in the order in which they were generated (but other heuristics generated in between are omitted due to space constraint).

\begin{tcolorbox}[title=Heuristics Generated throughout the Evolution Process (an Inconsecutive Excerpt), colback=gray!10, colframe=gray!50, coltitle=black]
{\footnotesize\ttfamily
→ return v\\
→ return v + pdf(v) - cdf(v)\\
→ return v + (cdf(v) / pdf(v)) - (1 / (1 + v))\\
→ return v - cdf(v) * pdf(v)\\
→ return v - (1 - cdf(v)) * v + (pdf(v) / (1 - cdf(v))) * v\\
→ return v - cdf(v) / pdf(v)\\
→ return v - (1 - cdf(v)) * v + pdf(v)\\
→ return v + (v - cdf(v) / pdf(v))\\
→ return v - cdf(v) / pdf(v) if v > 0 else -1\\
→ return v - pdf(v) * cdf(v) * 0.5\\
→ \textbf{return v - (1 - cdf(v)) / pdf(v)  [Successfully generate the virtual valuation function]}\\
→ return v - pdf(v) / cdf(v) + 1 / (1 - cdf(v))  [Continue exploring more options]\\
→ return v - v * (1 - cdf(v)) / pdf(v)\\
→ return v - (1 - v) * pdf(v) / (1 - cdf(v)) + v * (1 - v) * (1 - cdf(v)) / pdf(v)\\
→ ... [Explore more, then realize virtual valuation is the best and finally converge to it]
}
\end{tcolorbox}

The results indicate that GPT-3.5-Turbo does not have knowledge of the concept of virtual valuation or is unable to associate our problem directly with such knowledge. It starts the function searching process with the aim of maximizing the score (which is the seller's revenue in this setting), generating various functions, including different combinations of v, cdf(v), and pdf(v).

We also tried different distributions and different objectives, e.g., asking the AMD framework to generate a heuristic that simultaneously works well on multiple distributions. Details can be found in Appendix~\ref{apdx:rediscover}.

When we use more capable models such as DeepSeek-V3, DeepSeek-R1, 
Llama-3.1-70B-Instruct, and GPT-4o models in our AMD framework, they all output the virtual valuation function in one or two iterations. When we further ask them to explain their solutions, they will explain why Myerson's results can be applied to this setting. The following is an excerpt from the output. 
\begin{Verbatim}[breaklines=true,fontsize=\scriptsize,breaksymbol={}]
    The key insight from Myerson's theory is that the optimal auction allocates the item to the bidder with the highest virtual value, provided the virtual value is non-negative. ... This approach ensures that the mechanism is both incentive compatible and individually rational, as it adheres to the principles of Myerson's optimal auction. ...
\end{Verbatim}

Thus, in these cases, it seems that, unlike GPT-3.5-Turbo, the more capable models directly tied the problem to their prior knowledge about Myerson's auction; of course, the fact that they can do this would be of little help when trying to solve a previously unsolved mechanism design problem.

Another interesting finding worth noting is that when we remove the first sentence of the system prompt, that is ``You are a mechanism design expert.'', then GPT-4o loses its capability to associate the problem with  Myerson's existing results. When this sentence is removed, GPT-4o will not immediately output the virtual valuation. Instead, it will work similarly to GPT-3.5-Turbo, searching different functions and gradually evolving towards a better heuristic. 

All experiments presented in this section are without pre-set evolution strategies. We also tried different pre-set evolution strategies; details deferred to Appendix~\ref{apdx:rediscover}. 
More experiments and details for this section can also be found in Appendix~\ref{apdx:rediscover}. 


\section{An Application to Social Welfare Maximization Problems}\label{sec:redistribution}

We apply our framework to welfare-maximizing mechanism design. As an example, we tackle the problem of finding optimal VCG redistribution mechanisms with multiple identical items. 

\subsection{Optimal VCG redistribution mechanism}
We consider the setting of allocating one or more items among several agents. 
For this setting, one well-known mechanism with desirable properties is the Vickrey-Clarke-Groves (VCG) mechanism, which is efficient, strategy-proof, individually rational, and does not incur a deficit~\cite{vickrey1961counterspeculation,clarke1971multipart,groves1973incentives}. However, it is not strongly budget balanced, i.e., it incurs a positive total payment from the agents in general.  When there is no party to receive this payment (such as a seller), then this payment is ``burned'' in the VCG mechanism. 
The aim of a VCG redistribution mechanism is to redistribute some portion of the collected VCG payments back to the agents, while maintaining the desirable properties of VCG mechanism.\footnote{Note that by ``the VCG mechanism'' we mean the Clarke mechanism; the resulting redistribution mechanism is generally still a Groves mechanism.} 
We require that agent $i$'s redistribution amount received be independent of her own bid to ensure strategy-proofness. In other words,
agent $i$'s redistribution takes the form of $r_i(b_{-i})$, where $r_i$ is referred to as the
redistribution function and $b_{-i}$ denotes the other agents' bids.
For anonymous settings, we omit the subscript and simply use $r(b_{-i})$ to denote the
redistribution amount. The function $r$ is what we aim to evolve using LLM.
Our objective is to get as close to budget balance as possible without ever incurring a deficit~\cite{cavallo2008efficiency,guo2011vcg,guo2016competitive,guo2007worst,guo2008better}. In this paper's setting, we aim to redistribute as much as we can {\em in expectation} (i.e., maximizing $\sum_ir(b_{-i})$ in expectation), while ensuring that, for all bid profiles, $\sum_ir(b_{-i})$ is never more than the total VCG payment.

The design criteria for this problem include strategy-proofness (SP), individual rationality (IR), feasibility, and weakly budget balance (WBB). The design objective is to maximize the expected total redistribution while satisfying these criteria. 

We consider the case where there are multiple bidders and multiple units of an item, and each bidder wants at most one unit of the item.

\subsection{Implementation and the Waterfilling Fix}

We designed specifications for the VCG redistribution mechanism setting described in the previous paragraph. 
We perform a fix to ensure that any mechanism generated by an LLM in our framework meets all design criteria after the fix. 
\begin{itemize}
    \item In our VCG redistribution mechanism,  the individual rationality criterion is satisfied by setting the amount redistributed back to each bidder to be nonnegative. 
    \item Strategy-proofness is guaranteed by limiting the redistribution function to depend only on the bids of other bidders, not on the bidder's own bid. This restriction guarantees that a bidder manipulating his or her own bid will not influence the redistribution value he or she received, eliminating incentives for strategic manipulation. 
    \item No item being allocated to multiple bidders is immediately enforced by using the VCG mechanism, which inherently ensures valid allocations.
    \item To ensure that the total payment (weakly) exceeds the total redistribution, we apply a ``waterfilling''-style post-processing fix. 
\end{itemize}

Code Specification~\ref{spec:vcg-redistri} shows the code specification for AMD-LLM workflow on the VCG redistribution mechanism design setting.
In this code snippet, the heuristic function represents the value of the redistribution to each bidder, which takes other bidders' bids (sorted) as input and outputs the redistribution value. This function will be evolved by our AMD workflow.

\begin{algorithm}
\caption{Waterfilling Fix for VCG Redistribution Mechanism}
\begin{algorithmic}[1]\label{code:waterfilling} 
\STATE \textbf{Input:} $redistri\_vec$: redistribution vector, $payment\_vec$: payment vector
\STATE \textbf{Output:} $fixed\_redistri\_vec$: fixed redistribution vector
\STATE \textbf{Procedure:}
\STATE Initialize $fixed\_redistri\_vec \leftarrow redistri\_vec$
\STATE Get the shape of $redistri\_vec$: $n$

\STATE Calculate overflow: $overflow \leftarrow - \sum(payment\_vec[i]) + \sum(redistri\_vec[i])$
\IF{$overflow < 0$}
    \STATE $fixed\_redistri\_vec \leftarrow redistri\_vec$

\ELSE
\STATE Sort $redistri\_vec[i]$ to get $sorted\_vals$ and $sorted\_indices$

\FOR{$j = 0$ to $n - 1$}
    \IF{$sorted\_vals[j] \cdot (n - j) > overflow$}
        \STATE Create $subtract\_tensor$ with size $sorted\_vals$
        \STATE Set $subtract\_tensor[j:] \leftarrow overflow / (n - j)$
        \STATE Update $sorted\_vals \leftarrow sorted\_vals - subtract\_tensor$
        \STATE \textbf{break}
    \ELSE
        \STATE $overflow \leftarrow overflow - sorted\_vals[j]$
        \STATE $sorted\_vals[j] \leftarrow 0$
    \ENDIF
\ENDFOR

\STATE Create $reconstructed\_x$ with the size of $sorted\_vals$
\STATE Map $sorted\_vals$ back to original order: $reconstructed\_x[sorted\_indices] \leftarrow sorted\_vals$
\STATE Update $fixed\_redistri\_vec[i] \leftarrow reconstructed\_x$
\ENDIF

\RETURN $fixed\_redistri\_vec$

\end{algorithmic}
\end{algorithm}

We perform a waterfilling fix for cases where the total redistribution value exceeds the total payment value. The idea is to distribute (i.e., deduct from the redistribution amount) the excessive amount among all agents evenly. If it is impossible to distribute evenly due to some bidders already getting too little redistribution, then take all the redistribution amount from these bidders and continue distributing the outstanding amount from the remaining bidders, again, evenly. Therefore, when the total redistribution is greater than the total payment, the waterfilling fix will always reduce the total redistribution to exactly match the total payment, so that weakly-budget-balanced (WBB) is satisfied.
See Algorithm~\ref{code:waterfilling} for more details.

Let $\text{wf}(b_i, b_{-i})$ be the {\em fix value} for bidder $i$'s redistribution value from the waterfilling fix process, which is the difference between the output of the heuristic function and the value returned by Algorithm~\ref{code:waterfilling}. We want our redistribution mechanism with fixed redistribution values to still be strategy-proof, so we require that the fix value be independent of bidder $i$'s value $b_i$. We thus propose a (further corrected) fix that reduces the redistribution value by the maximum of the fix values (as previously specified) over all possible $b_i$'s, i.e., 
\[\text{fix}_i(b_{-i})=\max_{\forall b_i'\in[0,1]} \text{wf}(b_i', b_{-i}).\]
Therefore, the fixed redistribution value equals the original redistribution value given by the heuristic function, minus this maximum fix value over all possible bids, i.e., 
\begin{align}
    \text{fixed\_redistri}_i(b_{-i})=\text{heuristic}(b_{-i})-\max_{\forall b_i'\in[0,1]} \text{wf}(b_i', b_{-i}). \label{eq:fixed-redistri}
\end{align}

\begin{proposition}  
    After applying the (further corrected) waterfilling fix, the fixed version of our redistribution mechanism with any LLM-generated heuristic function becomes feasible, individually rational (IR), strategy-proof (SP), and weakly budget balanced (WBB). The resulting fixed mechanism has a redistribution value $fixed\_redistri$, defined in Eq.~(\ref{eq:fixed-redistri}). 
\end{proposition}

\begin{proof}
We consider the properties respectively. 
\begin{itemize}
    \item The (further corrected) fix ensures that the payment for bidder $i$ is independent of bidder $i$'s bid, and thus it maintains the strategy-proofness property of the original redistribution mechanism (the mechanism before the fix). 
    \item IR is preserved because in the waterfilling process, the maximum amount of the fix is just the redistribution value from the original mechanism, and thus this will not lead to a negative redistribution. 
    \item Feasibility is guaranteed by the VCG mechanism we use, and the waterfilling fix  will not change the allocation function. 
    \item WBB is met if the total redistribution after the fix is no larger than the total payment. We only need to consider the case where there is a positive fix. In such a case, 
    \[\sum_{i}\text{wf}(b_i, b_{-i})\leq \sum_{i}\max_{\forall b_i'\in[0,1]} \text{wf}(b_i', b_{-i})=\sum_{i}\big(\text{heuristic}(b_{-i}) - \text{fixed\_redistri}_i(b_{-i})\big). \]
    Combined with the facts that Total Redistribution before Fix - Total Payment $=\sum_{i}\text{wf}(b_i, b_{-i})$, and $\text{Total Redistribution before Fix} = \sum_{i}\text{heuristic}(b_{-i})$, we can conclude that 
    \[\sum_{i}\text{fixed\_redistri}_i(b_{-i})\leq \text{Total Payment}. \]
\end{itemize}
\end{proof}

In Code Specification~\ref{spec:vcg-redistri}, the main function is used to evaluate the performance of the currently selected heuristic, whose return value is used as the fitness in the evolutionary process. 
In the implementation, we sample a batch of bids (we set batch\_size = 3,000 for the experiments in this section) from the valuation distribution to diminish the influence of randomness on our AMD evolution workflow. 
Lines 7-9 call the predefined vcg\_payment function to calculate the VCG payment values, detailed in Appendix~\ref{apdx:redistribution}. Lines 10-12 calculate the redistribution value using the heuristic function generated and evolved in our AMD-LLM workflow. Lines 13-14 perform the waterfilling fix that turns the mechanism based on the LLM-generated heuristic (possibly not WBB) into a mechanism that satisfies all the design criteria. 
More details are included in Appendix~\ref{apdx:redistribution}. 

\begin{lstlisting}[basicstyle=\ttfamily\scriptsize,mathescape=true,caption={Specification for VCG Redistribution Mechanism Design},label={spec:vcg-redistri}]
import torch
import funsearch
@funsearch.run    # Evaluation function to be executed by FunSearch
def main() -> float:
  batch_size, num_items, num_bidders = 3000, 2, 4
  batched_bids = get_bidders_values(batch_size, num_bidders)
  batched_payment = vcg_payment(batched_bids, num_items) #Calcualte VCG payment
  batched_allocation = (batched_bids > batched_payment).float()
  batched_payment = batched_payment * batched_allocation
  batched_redistri = torch.zeros_like(batched_bids)
  for batch_id in range(batch_size):
    batched_redistri[batch_id] = redistribution_func(batched_bids[batch_id], num_items, num_bidders)  #redistribution_func will call the heuristic
  delta_redistri = fix_waterfilling(batched_bids, num_items, num_bidders)
  batched_redistri_fixed = batched_redistri - delta_redistri
  score = torch.mean(torch.sum(batched_redistri_fixed, dim=1)).item()
  return score
  
def redistribution_func(bids, num_items, num_bidders):
  redistri_vec = torch.zeros_like(bids)
  for bidder in range(num_bidders):
    others_bids = torch.concat((bids[:bidder], bids[bidder+1:]))
    heu_out = heuristic(torch.sort(others_bids)[0])
    redistri_vec[bidder] = heu_out
  return redistri_vec
  
@funsearch.evolve     # The heuristic function to be evolved through FunSearch
def heuristic(bids): 
  return 0    # Start with a naive heuristic
\end{lstlisting}

We also experiment with the  NNA-AMD-LLM workflow. 
We train a RegretNet~\cite{RochetNet} to function the same as the heuristic function in the workflow. The input to the neural network consists of the other bidders' bids (sorted) and the output is the redistribution value. 
The unsupervised training objective is to maximize the total redistribution while penalizing on feasibility, SP, IR, and WBB violations. 
Thus, we define the loss function to be minimized during training to be
\begin{align}
    Loss = -\text{score}+\alpha \cdot \sum_i\text{Regret(i)}+\beta\cdot f\text{(violations)} \label{eq:loss-redistri}
\end{align}
where the score is the total redistribution value in this case. Coefficients $\alpha, \beta$ are adjusted during training. Regret($i$) is the ex post expected regret, which measures how much the mechanism violates SP:
$\text{Regret(i)} = \mathbb{E}\big[\max_{b_i'} u_i(b_i',b_{-i})-u_i(b_i, b_{-i})\big]$,
where $u_i$ is bidder $i$'s utility function, i.e., bidder $i$'s valuation of the item minus payment if allocated an item, plus the received redistribution value. 
We also perform the same waterfilling fix to the neural network output in case it is not WBB. 

We then use such a neural network in NNA-AMD-LLM. In this Program-by-Example workflow, we want the LLM to generate a heuristic that approximates the NN's pattern. 
The code specification we use is almost the same as the one used in AMD-LLM, except that the evaluation score is replaced by directly calculating the difference between the heuristic output and the NN output on the samples.  

The system prompt we use in our framework is shown in Prompt~\ref{prompt:redistribution}, in which we include the problem setting and some requirements for designing the mechanism. 
We use the same prompt for our AMD workflows with and without neural network assistance.

\begin{mybox}[label=prompt:redistribution]{Prompt \thetcbcounter. The System Prompt for VCG Redistribution Mechanism Design}
\begin{Verbatim}[breaklines=true,fontsize=\scriptsize,breaksymbol={}]
You are a mechanism design expert. Your task is to design a heuristic function for a redistribution mechanism. The heuristic function's input is only others_bids (a 1-dimension np.array of other bidders' bids, except a certain bidder's), and the output is a float number representing the amount to allocate back to the certain bidder. The goal of a VCG redistribution mechanism is to allocate resources efficiently, ensuring that participants have incentives to bid truthfully (incentive compatibility) and are not worse off by participating (individual rationality). Additionally, it aims to redistribute part of the payment surplus back to participants, minimizing the revenue retained by the auctioneer or central authority. This redistribution enhances fairness and can increase participant satisfaction while maintaining the efficiency and truthfulness properties of the VCG mechanism. In this particular setting, we are targetting 2 items and 4 bidders. The input to the heuristic function is only others_bids (the other bidders' bids), which is a array of length 3. There are no other args as input, do not use other parameters without defining. If you want to use any new parameters in addition to others_bids, you must initialize that explicitly in the code of heuristic function. You should remove redundant code (i.e., if a shorter version of code has equivalent function, please use the shorter one). You only need to design one new heuristic function. Note: use 2 spaces as indent for Python code. Only output a standalone heuristic_v{version} function code, do not output anything else. You can use various functions that commonly occurs in game theory and mechanism design literature. You should make the code concise and short. Please specify any parameter not defined before that is used in the heuristic function. Please use np for mathematics computation.  
\end{Verbatim}
\end{mybox}

\subsection{Experimental Results}\label{sec:redistribution-experiment}
In our experiment, we set the number of bidders to be 4 and the number of (identical) units of the item to be 2. All bidders' valuations are sampled i.i.d.\ from the uniform distribution $U[0, 1]$. 
The experiments in this section use the DeepSeek-V3 model. 
We initialize the heuristic to na\"ively returning 0. 

\subsubsection{Comparison of Methods}

There are existing manual VCG redistribution mechanisms, such as 
\cite{cavallo2006optimal,guo2008better}. These two works each give the same mechanism for the setting we consider here, which is
$\text{heuristic}(b_{-k}) = 0.5 \min(b_{-k})$.
Under the distribution we use, this heuristic achieves an expected total redistribution value of 0.5. 

See Table~\ref{table:redistri-scores} for a comparison of the optimal mechanisms found by different methods. The score is from an average of 3,000 test samples drawn from the distribution. 
From the results, we can see that our AMD-LLM workflow can find a mechanism with better performance than the existing manual mechanisms.
With the assistance of NN, our AMD workflow can identify better mechanisms compared to those found without NN. However, there remains a gap between the results from NNA-AMD-LLM and the results produced directly by the NN. 
Nevertheless, we emphasize that while the results produced directly by neural networks (NN) are black-box models, our solutions, which are concise Python code, offer the advantage of greater interpretability, as discussed in detail below.


\vspace{-0.3cm}
\begin{table}[htbp]
\centering
\small
\caption{Comparison of Average Test Scores (Expected Total Redistribution) of the Optimal Mechanisms Found by Different Methods}\label{table:redistri-scores}
\begin{tabular}{lc}
\toprule
\textbf{Method} & \textbf{Avg Test Score} \\
\midrule
\cite{cavallo2006optimal,guo2008better} & 0.4935 \\
(ours) AMD-LLM & 0.5838 \\
(ours) NNA-AMD-LLM & 0.5917 \\
RegretNet~\cite{RochetNet} & 0.6254 \\
\bottomrule
\end{tabular}
\end{table}
\vspace{-0.4cm}

\subsubsection{Evolution dynamics and examples of generated heuristics}

Fig.~\ref{fig:redistri-pbe-1} shows the result from NNA-AMD-LLM. The $x$-axis gives the evolution iteration, and the $y$-axis gives the score. The score is the negation of the average squared difference between the heuristic's output and the NN's output on the same set of samples. A larger score (smaller absolute value of the score) means the heuristic better approximates the NN. 
The blue curve represents the highest average score achieved from iteration 0 to the most recent iteration through the evolutionary process, using NNA-AMD-LLM in the context of VCG redistribution mechanism design.
The red curve in Fig.~\ref{fig:redistri-pbe-1} depicts the best score within the last five iterations, where five corresponds to the number of predefined evolution strategies we employ. These strategies are applied sequentially in a repeating cycle, meaning each consecutive set of five iterations is generated by five distinct evolution strategies.
The red curve oscillates and temporarily decreases as the process explores different heuristics. 
The blue boxes highlight sample heuristics generated at specific stages of the evolutionary process, as indicated by the blue arrows. Some blue arrows are not pointing at the blue curve, which means that the heuristic is not the best so far; for these heuristics, the arrows are pointing to the corresponding iteration but not necessarily the value of its score.  

From this result, we can observe that our method has the following capabilities. 
\begin{enumerate}
    \item It can experiment with existing methods by updating the parameters. See Fig.~\ref{fig:redistri-pbe-1} as an example, and consider the sampled outputs at iterations 268 and 348 (as well as other outputs not shown in the figure). During the exploration process, the parameters in the heuristics are adjusted. Specifically, the coefficient for harmonic\_mean is updated from 0.4 (at iteration 268) to 0.3 (at iteration 348). 
    \item It can make slight adjustments to the content in heuristics. For instance, the min\_bid parameter (at iteration 268) is replaced with median\_bid (at iteration 348) to facilitate exploration. 
    \item In addition to refining and building upon previously successful mechanisms, the approach is also capable of generating entirely novel ideas. As illustrated by the sample heuristics from iterations 0, 65, 268, and 977, the evolutionary process proposes heuristics with diverse and distinct structures.
    \item It can identify elegant heuristics that have been demonstrated to be effective in the prior literature. Iteration 65 successfully discovers the redistribution mechanism proposed by~\cite{cavallo2006optimal,guo2008better}. 
\end{enumerate}

\begin{figure}[htbp]
    \centering
    \includegraphics[width=\textwidth]{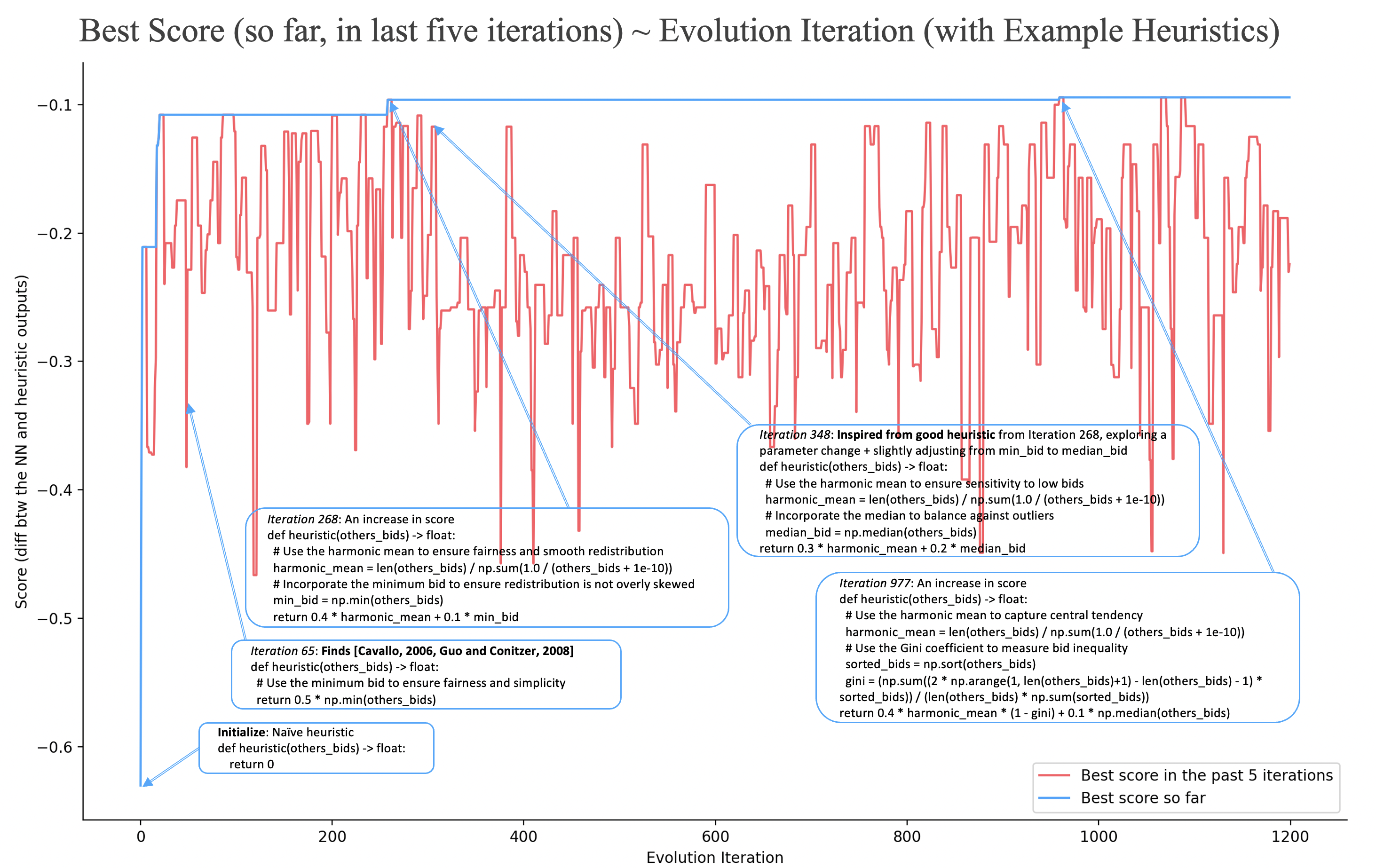}
    \caption{The best average score (so far, in last five iterations) by evolution iteration, using NNA-AMD-LLM, on VCG redistribution mechanism design. The blue boxes contain some example heuristics generated, from the corresponding positions in the evolutionary process, pointed at by the blue arrows.}
    \label{fig:redistri-pbe-1}
\end{figure}

The results of AMD-LLM have evolution patterns similar to those of NNA-AMD-LLM as presented in this section, and therefore we include the details of those results in Appendix~\ref{apdx:redistribution}.

\section{An Application to Revenue Maximization Problems}\label{sec:correlated}

We now move on to another application: maximizing revenue, specifically in the context of correlated bidders.

\subsection{Single-item auction with multiple correlated bidders}
In this auction, multiple correlated buyers compete for a single item. This setting is more complex than the independent private values model and can lead to different optimal auction designs. 

We assume that bidders' valuations are drawn from a joint probability distribution. That is, the bidders' valuations are correlated. 
The optimal auction design can be significantly different from the independent case, because the seller may be able to extract more revenue by leveraging the correlation information. As a result, standard auction formats like second-price auctions are no longer optimal in general~\cite{dughmi2014sampling,morgenstern2016learning,dobzinski2015approximately}. In fact, \citet{papadimitriou2011optimal} showed that optimal
single-item auction design under correlation is NP-hard to approximate.

\subsection{Implementation and the Monotonicity Fix}
We experiment in the revenue-maximizing auction setting with 1 item and 2 bidders, where the two bidders' values are correlated. 
The heuristic function's input consists of the bidders' bids (a vector of length 2). Its output is the allocation function, which is a vector of length 3, where each element should be in range [0, 1]: if the first index of the output vector is the largest, then allocate the item to the first bidder; if the second index is the largest, then allocate to the second bidder; if the third index is the largest, then do not allocate to any bidder. The goal of the mechanism designer is to maximize the revenue while ensuring feasibility, strategy-proofness, and IR. 
We use the same pre-set evolution strategies as in Sec.~\ref{sec:redistribution-experiment}. We use a grid-like correlated distribution in the experiments, detailed in Appendix~\ref{apdx:correlated}. The experiments in this section use the DeepSeek-V3 model. 

For AMD-LLM, the score used in the evolutionary process is the average total payment. The specification code we use is shown in Code Specification~\ref{spec:example-one-item}.
For NNA-AMD-LLM, the score used in the evolutionary process is the average L2 difference between the output of the heuristic function and the output of the trained RegretNet.   

\textit{The fixing process. } 
In this section, we use the monotonicity fix along with the corresponding payment function derived from the allocation heuristic explained in Example~\ref{example:auction-spec}. This fix ensures that the allocation function is monotone in each bidder's valuation, and thus ensures SP. 

\subsection{Experimental Results}

In the correlated-bidders auction setting, the experiments demonstrate that our frameworks achieve competitive performance while providing interpretable solutions. The AMD-LLM frameworks, both with and without NN assistance, perform on par with Myerson's mechanism with ironing, a well-established and fairly intricate mechanism (noting that ironing involves finding the convex hull of the revenue-demand curve); see Table~\ref{table:correlated-scores}. 
Notably, while neural-network-based methods (RegretNet) achieve slightly higher performance, they lack interpretability. In contrast, our frameworks generate mechanisms that are both interpretable and near-optimal, enabling a deeper understanding of the underlying logic and making them more practical for real-world applications. This balance of performance and transparency highlights the strength of our approaches in delivering efficient and trustworthy solutions for mechanism design.
The evolution dynamics is similar to the one in Fig.~\ref{fig:redistri-pbe-1}. Details are included in Appendix~\ref{apdx:correlated}.

\begin{table}[htbp]
\centering  
\small
\caption{Comparison of Average Test Scores (Expected Total Payment) of the Optimal Mechanisms Found by Different Methods}\label{table:correlated-scores}
\begin{tabular}{lc}
\toprule
\textbf{Method} & \textbf{Avg Test Score} \\
\midrule
Myerson with Ironing~\cite{Myerson81:Optimal} & 0.3857 \\
(ours) AMD-LLM & 0.3857 \\
(ours) NNA-AMD-LLM & 0.3857 \\
RegretNet~\cite{RochetNet} & 0.3878 \\
\bottomrule
\end{tabular}
\end{table}

\paragraph{Insights given by LLM-generated heuristics during evolution}
The following is one LLM-generated heuristic at iteration 95 in the experiment using NNA-AMD-LLM. 
It shows that the LLM has learned that the mechanism may choose not to allocate in certain cases when doing so is unprofitable for the seller. 
The allocation function derived from the following LLM generated heuristic (with sigmoid function) is effectively equivalent to the reserve-price auction with a reserve price of 0.5. However, this LLM generated heuristic outperforms the corresponding reserve-price mechanism in our framework, because of the additional fixing process. The monotonicity fix derives payments from the allocation function for both the 0.5-reserve-price mechanism and the LLM generated sigmoid-involved mechanism. Although the two mechanisms has the same allocation function, the payment derived are different, and thus, the probability of selling an item is also different (an item is sold if and only if the derived payment from the monotonicity fix is no greater than the bidder's valuation for the item).

{\footnotesize
\begin{verbatim}
def heuristic(bids):
  threshold = 0.5
  sigmoid = lambda x: 1 / (1 + np.exp(-10 * (x - threshold)))
  alloc_bidder1 = sigmoid(bids[0])
  alloc_bidder2 = sigmoid(bids[1])
  no_alloc = 1 - max(alloc_bidder1, alloc_bidder2)
  return [alloc_bidder1, alloc_bidder2, no_alloc]
\end{verbatim}
}

One interesting observation is that, in spite of the mechanism design domain, LLM-generated heuristics often include functions whose usage is common in AI, such as sigmoid, KL-divergence, and convolution. While a combination of these functions can perform well, it may not be the optimal choice. One hypothesis is that an LLM trained less on these functions might yield better performance.
See Appendix~\ref{apdx:correlated} for more experimental results.

\section{Conclusion and Future Work}\label{sec:conclusion}

In this paper, we introduced a novel framework that reformulates mechanism design as a code generation task, leveraging LLMs to automate mechanism design. By generating and evolving heuristic mechanisms described in code, with the additional problem-specific fixing process, our approach bridges the gap between traditional analytical methods and modern automated techniques. This framework ensures that the generated mechanisms are interpretable, and adhere to key design criteria such as feasibility and strategy-proofness.
Our experiments demonstrated that LLM-generated mechanisms achieve competitive performance while offering greater interpretability compared to existing neural-network-based approaches. Notably, our framework rediscovered well-known mechanisms, such as Myerson's optimal auction, and provided insights into NN-based solutions like RegretNet. This highlights the potential of LLMs to automate mechanism design while enhancing the transparency and interpretability of the generated mechanisms.  The effectiveness of the methodology is likely to increase as LLMs and associated techniques improve, thereby also being able to scale to societal-scale problems that are inaccessible to traditional mechanism design.  Meanwhile, unlike alternative neural-network-based approaches, the mechanisms produced are more interpretable and thereby easier to analyze, allowing us to ensure that their deployment is safe.

There are several avenues for future research. 
As mentioned in the experimental sections, LLM-generated heuristics often rely heavily on functions that are commonly used in AI, which may not be ideal for mechanism design. Future work could explore using an LLM less trained on AI-style code, or fine-tuning it on mechanism design and economic theory to better align with our framework.

LLM-generated heuristics provide valuable insights, which could also be leveraged to refine system prompts automatically. Recent works like~\cite{fernando2023promptbreeder} focus on improving prompts, and integrating such methods into our AMD framework could enhance its effectiveness.

Incorporating human feedback into the evolutionary process can yield more intuitive and practical mechanisms, better suited to real-world constraints. Additionally, combining LLMs or neural networks with various symbolic reasoning techniques could further improve interpretability and robustness.

\section{Acknowledgment}

We thank Thomas Kleine Büning, Edward Hughes, Sai Srivatsa Ravindranath, and Emanuel Tewolde for valuable discussions. We also thank Yee Man Choi and Fei Fang for their contributions to discussions and experiments related to LLM code generation. 
Additionally, we are grateful to the Cooperative AI Foundation, Polaris Ventures (formerly the Center for Emerging Risk Research) and Jaan Tallinn’s donor-advised fund at Founders Pledge for financial support.

\bibliographystyle{ACM-Reference-Format}
\bibliography{my-bibliography}

\appendix

\section{More Experiments and Implementation Details on Rediscovering Virtual Valuation}\label{apdx:rediscover}

\subsection{Hyperparameters for the Evolutionary Algorithm}  
The evolutionary algorithm uses islands—independent subpopulations of solutions—to maintain diversity and explore different regions of the solution space. Each island evolves separately through selection, crossover, and mutation, with better-performing solutions more likely to reproduce.

The main hyperparameters include:
\begin{itemize}
    \item num\_islands: how many islands to maintain, i.e., the number of subpopulation. An island refers to a subpopulation of solutions that evolve independently. More islands help maintain higher population diversity, but will converge slower. 
   \item reset\_period: how often the weaker islands should be reset. The evolutionary algorithm periodically reinitializes weaker islands to prevent stagnation. 
    \item cluster\_sampling\_temperature: controls the exploration-exploitation trade-off when selecting clusters within an island. Higher temperatures encourage more exploration, while lower temperatures favor exploitation of better-performing clusters. This parameter will decay gradually, shifting the sampling behavior from exploration to exploitation over time.
    \item functions\_per\_prompt: how many samples of previous generated programs to include in prompts. 
\end{itemize}
For different problems, we set different hyperparameters for the evolutionary updates. 

In our experiments on rediscovering virtual valuation, we set number of islands to be 10, reset period to be an hour. The cluster sampling temperature is initialized to be 0.1 and will decay linearly, more specifically, will decay following 0.1(1 $-$ \#programs / 30000). 

\subsection{Simpler Setting with only 1 bidder}
In this case, the evolution process can find the following heuristic faster than finding the original virtual valuation function (with division by pdf(v)). 
{\small
\begin{verbatim}
def heuristic(v, alpha, beta): 
  return v * pdf(v) - cdf(v) + 1
\end{verbatim}
}
This heuristic is actually equivalent to virtual valuation function in one-bidder case, because we are comparing the virtual valuation with 0 to decide whether or not to allocate. 
Whether it is greater than zero remains the same if we divide this heuristic by pdf(v), which is then exactly the virtual valuation. 

However, this only applies to the one-bidder case, so it is not as generalized as the original virtual valuation function that can also be found through our AMD framework when we set the number of bidders to be at least 2.

\subsection{Combining Different Distributions in Evaluation}
In some cases, it is possible that the framework finds virtual valuation and then diverts to another heuristic with a slightly higher score. This is because we are evaluating the performance on 3,000 random samples from the distribution, and when the distribution is nontrivial and there are more bidders, the variance can be large. As a result, it may ``overfit'' to a heuristic that is suboptimal in general. The following shows one such overfitted heuristic, which has the best score but is not optimal in theory.  
{\small
\begin{verbatim}
def heuristic(v, alpha, beta): 
  return v ** 2 * pdf(v) / (1 - cdf(v)) - abs(v - 0.5) + v - 0.5 
\end{verbatim}
}

We then tried different distributions and different objectives. 
We still use the GPT-3.5-Turbo model in our experiment.
The distribution we tried includes the uniform distribution $U[0,1]$, the Beta distributions Beta(2, 5), Beta(0.5, 0.5), and Beta(2, 2). The results are similar to those presented in Sec.~\ref{sec:rediscover} under all these distributions. 
That is, it is possible to find the optimal auction rule, while it also is possible that our framework overfits to the distribution and generates a heuristic that works well in that specific distribution we use. 

To address this issue, we further tried asking the AMD framework to generate a heuristic that works well simultaneously on multiple distributions. 
In the evaluation function, we draw samples from different distributions, e.g., 3000 samples from each of the Beta(2, 5), Beta(0.5, 0.5), and Beta(2, 2) distributions, and then calculate the average score for each distribution. We aim to optimize the three scores simultaneously. 
However, if we fix the rule that combines these three scores, then we are still optimizing over some certain metric. It is similar to optimizing over one certain distribution, and the overfit issue still persists. 

\subsection{Pre-set Evolution Strategies}
The experiments in Sec.~\ref{sec:rediscover} are using our framework without pre-set evolution strategies. 
We also tried using problem-specific evolution strategies. The five evolution strategies used are as follows. Some of them are also used by~\citet{liu2024evolution}. 
\begin{enumerate}
\ttfamily\small
    \item Please help me create a new heuristic that has a different arithmetic operator or a different parameter from the given heuristics. 
    \item Please help me create a new heuristic that has a totally different form from the given ones but can be motivated from them. 
    \item Please assist me in creating a new heuristic that has a different form (ideally a shorter form) but can be a modified version of the heuristic provided.
    \item Please help me create a new heuristic by using a different arithmetic operator, or switching the two operands from the given heuristic.
    \item First, you need to identify the main components in the heuristic function below. Next, analyze whether any of these components can be overfitted to the in-distribution instances. Then, based on your analysis, simplify the components to enhance the generalization to potential out-of-distribution instances. 
\end{enumerate}
\normalfont

In each iteration, one of these strategies is prompted to LLM together with the system prompt and the sampled previously generated heuristic functions. With the help of these heuristics, it is easier for our AMD framework to find the virtual valuation function. 

However, these five evolution strategies are specially designed to our problem of finding the existing virtual valuation function. This is to some extent too deliberate, because we already know the objective beforehand. 

For NNA-AMD-LLM, we tried using the LLM to generate a heuristic function that approximates the fixed RegretNet, as well as the unfixed RegretNet result, respectively. The performance is similar in both experiments. This could be because the LLM generated mechanism is still far from the NN ones, so the waterfilling process fixes by a large extent even if the framework evolves towards the fixed NN. 

\subsection{Other System Prompts That Help}
Meanwhile, we tried other system prompts that are specially designed to fit this problem of rediscovering virtual valuation. By giving some extra hints to the LLMs, these prompts can sometimes help to rediscover the virtual valuation faster. 

\begin{enumerate}
\ttfamily\small
    \item You can use various functions that commonly occurs in game theory and mechanism design literature. 
    \item The value v should serve as an additive term in the heuristic, i.e., the heuristic function should start with v $-$ or v $+$. 
    \item You can also use some other terms, e.g., (1 - cdf(x)) is the complementary cumulative distribution function (CCDF) or the survival function of a random variable. 
    \item (We first predefine survival(v) function in the code) You can use the survival function of the distribution survival(v). 
    \item Since cdf(v) and survival(v) are complementary, you should use one of these two, but not both, in the heuristic. 
\end{enumerate}
\normalfont

\section{Experimental Details for Designing VCG Redistributon Mechanism}\label{apdx:redistribution}

\subsection{Implementation of the VCG Payment}
Before we apply the VCG redistribution mechanism, the payment and allocation are decided by the VCG mechanism. 
VCG mechanism awards goods or resources to the bidders who value them the most (maximizing social welfare) and charges each winner an amount equal to the externality they impose on others -- specifically, the difference in social welfare if they were not participating. See Code Specification~\ref{code:vcg-implementation} for the code implementation in Python, which we use together with the code in Code Specification~\ref{spec:vcg-redistri} to build our AMD with LLM workflow for designing the VCG redistribution mechanism. 

\begin{lstlisting}[mathescape=true,caption={Implementation of VCG Mechanism in Python Code},label={code:vcg-implementation}]
def vcg_payment(bids_vec, num_items):
  top_values, top_indices = torch.topk(bids_vec, num_items, dim=1)
  sorted_bids, _ = torch.sort(bids_vec, descending=True, dim=1)
  threshold_value = sorted_bids[:, num_items].unsqueeze(1)
  payment_outputs = torch.full_like(bids_vec, 9999)
  payment_outputs.scatter_(1, top_indices, threshold_value.expand_as(top_indices))
  return payment_outputs
\end{lstlisting}

\subsection{Implementation of the Waterfilling Fix}
The waterfilling fix is implemented according to Algorithm~\ref{code:waterfilling}. See the supplementary code materials for its detailed implementation. 
We use the functions for the waterfilling fix together with the code in Code Specification~\ref{spec:vcg-redistri} and Code Specification~\ref{code:vcg-implementation} to build our AMD-LLM workflow for designing the VCG redistribution mechanism. 

\subsection{More Details on Implementation}
In our experiments on rediscovering virtual valuation, we set number of islands to be 10, reset period to be an hour. The cluster sampling temperature is initialized to be 0.1 and will decay linearly, more specifically, will decay following 0.1(1 $-$ \#programs / 30000). The same hyperparameters are used for correlated bidders auction setting. 

Since the evaluation in each iteration is based on samples and we need to perform an evaluation for each generated heuristic, the evaluation takes long if the number of samples is large. Thus, to balance the response time of LLM and the evaluation time, we set the number of samples to be 3,000.  
In order to reduce the impact of variance on the performance of our method, we first fix the set of parameters during the evolution process and then test using random samples. 

The NN we use is a multilayer perceptron with 2 layers, each with 100 nodes. We use ReLU activation. It is trained for 200,000 epochs. The trained model is WBB on 96.5\% of the in total 3,000 test samples. 

\subsection{Pre-set Evolution Strategies}
The experiments in Sec.~\ref{sec:redistribution} are using our framework with pre-set evolution strategies. 
The five evolution strategies used are as follows, which explicitly ask the evolutionary process to follow the principles of natural selection, mutation, and crossover. They are slightly altered from those used by~\citet{liu2024evolution}. 
\begin{enumerate}
    \item Please help me create a new heuristic that has a totally different form from the given ones. 
    \item Please help me create a new heuristic that has a totally different form from the given ones but can be motivated from them.
    \item Please assist me in creating a new heuristic that has a different form (ideally a shorter form) but can be a modified version of the heuristic provided. 
    \item Please identify the key parameters in the heuristic function and assist me in creating a new heuristic function that has different parameter settings of the key parameters.
    \item First, you need to identify the main components in the heuristic function below. Next, analyze whether any of these components can be overfitted to the in-distribution instances. Then, based on your analysis, simplify the components to enhance the generalization to potential out-of-distribution instances.
\end{enumerate}
In each iteration, one of these strategies is prompted to LLM together with the system prompt and the sampled heuristics from the database that stores all previously generated heuristics. 

Unlike what we did in rediscovering virtual valuation, in this setting, we cannot design those deliberate evolution strategies because we do not know what the optimal heuristic looks like for this open problem. As a result, we can only use these general strategies to facilitate the evolution process.

\subsection{A Further Improvement: the Reverse-Waterfilling Fix}
After the waterfilling fix, the total redistribution value will likely be less than the total payment value. 
While our aim is to maximize the social welfare, i.e., to redistribute back as much as possible, it is natural to think of a way to redistribute back the updated surplus (the surplus after fix). 
In order to still preserve SP, IR, and feasibility, we can still use a waterfilling-style method (and then similarly take the max over all possible bidding values for each bidder) to evenly redistribute more back to bidders. 
In the setting we experiment on, the reverse-waterfilling gives a non-zero but marginal improvement. 
See Algorithm~\ref{code:reverse-waterfilling} and supplementary code materials for details. 

\begin{algorithm}
\caption{Reverse-Waterfilling for VCG Redistribution Mechanism} 
\begin{algorithmic}[1]\label{code:reverse-waterfilling} 
\STATE \textbf{Input:} $fixed\_redistri\_vec$: redistribution vector after the waterfilling fix, $payment\_vec$: payment vector
\STATE \textbf{Output:} $reverse\_fixed\_redistri\_vec$: redistribution vector after the reverse-waterfilling process
\STATE \textbf{Procedure:}
\STATE Initialize $reverse\_fixed\_redistri\_vec \leftarrow fixed\_redistri\_vec$
\STATE Get the shape of $fixed\_redistri\_vec$: $n$

\STATE Calculate surplus: $surplus \leftarrow \sum(payment\_vec[i])  \sum(fixed\_redistri\_vec[i])$
\IF{$surplus < 0$}
    \STATE $reverse\_fixed\_redistri\_vec \leftarrow fixed\_redistri\_vec$
\ELSE
\FOR{$j = 0$ to $n - 1$}
    \STATE $reverse\_fixed\_redistri\_vec \leftarrow fixed\_redistri\_vec + surplus / n$
\ENDFOR
\ENDIF

\RETURN $reverse\_fixed\_redistri\_vec$

\end{algorithmic}
\end{algorithm}

Note that we do not necessarily have to check whether or not $surplus \geq 0$ here in the reverse-waterfilling process, like what we did in the waterfilling fix, because the reverse-waterfilling process is performed after the waterfilling fix that already ensures $surplus \geq 0$. 

For each bidder $i$, we still need to take max over all his or her possible bids to make the additional fix value independent of bidder $i$'s own bid, so that SP can be guaranteed. 
\[\text{reverse\_fixed\_redistri}_i(b_{-i})=\text{heuristic}(b_{-i})-\max_{\forall b_i'\in[0,1]} \text{wf}(b_i', b_{-i}) + \max_{\forall b_i''\in[0,1]}\text{reverse\_wf}(b_i'', b_{-i})\]
where $\text{reverse\_wf}(b_i, b_{-i})$ is the amount of update for bidder $i$' fixed redistribution value in the reverse-waterfilling process, i.e., the difference between the output of Algorithm~\ref{code:reverse-waterfilling} and the corresponding output of Algorithm~\ref{code:waterfilling}. 

\subsection{Additional Experiment Results}

\subsubsection{Results on AMD-LLM}
Fig.~\ref{fig:redistri-1} shows the result from AMD-LLM. The $x$ axis gives the evolution iteration, and the $y$ axis gives the score. The score is the average total redistribution value evaluated on the 3,000 samples drawn from the distribution $U[0,1]$. 

Fig.~\ref{fig:redistri-1} shows the evolutionary dynamics in the experiment: best average score (so far, in last five iterations) with the evolution iteration, using AMD-LLM, on VCG redistribution mechanism design. 
The results are similar to the ones presented in Sec.~\ref{sec:redistribution-experiment}. 
\begin{figure}[htbp]
    \includegraphics[width=0.75\textwidth]{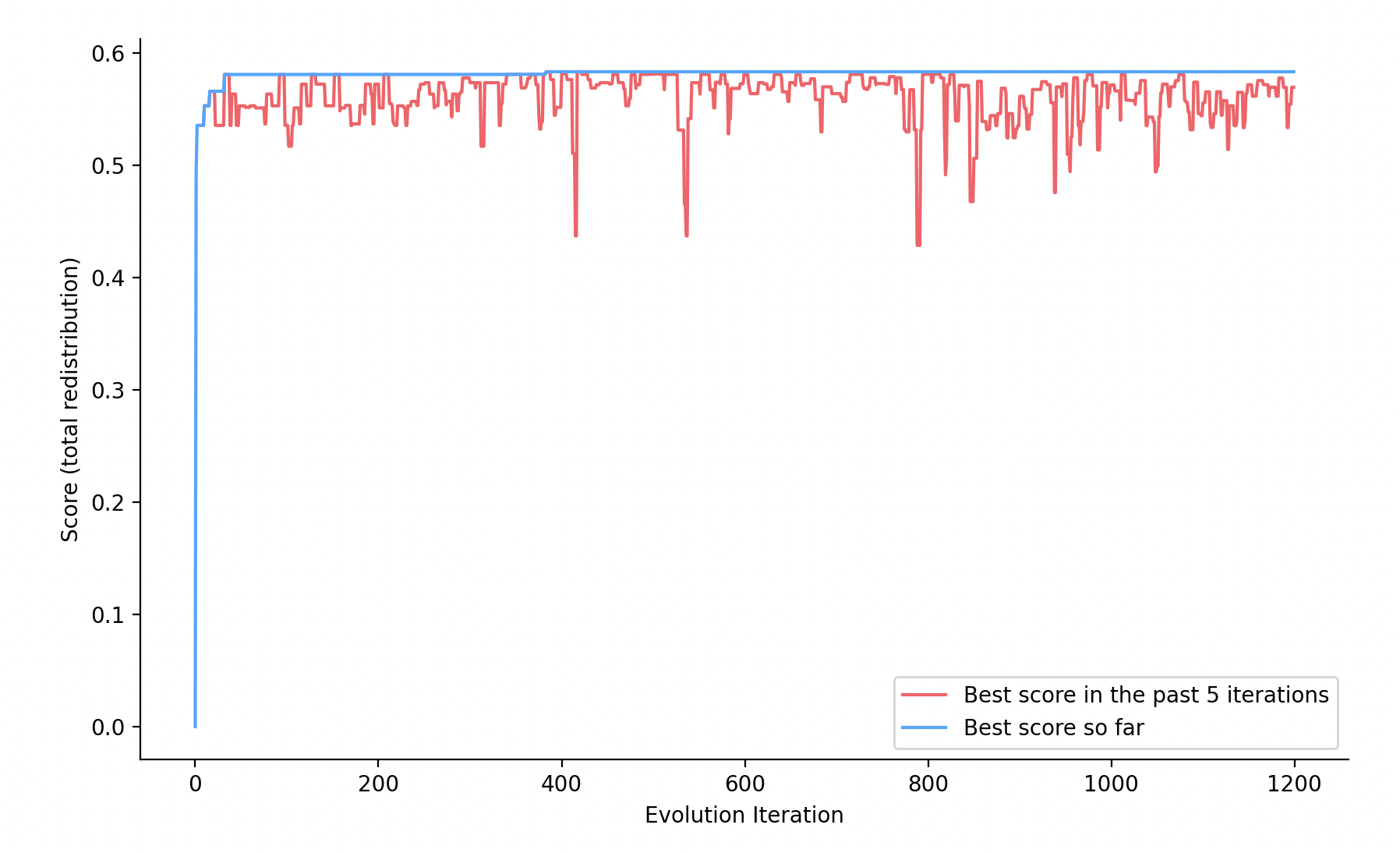}
    \caption{The best average score (so far, in last five iterations) with the evolution iteration, using AMD-LLM, on VCG redistribution mechanism design. }
    \label{fig:redistri-1}
\end{figure}

\subsubsection{Results on Other Models}

We also experimented with the Llama-3.1-70B-Instruct model for the same 4-bidder 2-item setting. 
The result is shown in Fig.~\ref{fig:fig_llama_redistri}. 

The $x$ axis gives the evolution iteration, and the $y$ axis gives the score (total redistribution value in our setting). The curve shows the trend of the current maximum score at each iteration, i.e., the best score so far, as the iteration goes on. The green boxes contain some example heuristics generated by the Llama model, from the corresponding positions in the evolutionary process pointed at by the green arrows. 

We run Llama for 4,000 iterations while the result from DeepSeek-V3 model presented in Sec.~\ref{sec:redistribution-experiment} is from fewer than 1,000 iterations. 
Compared with the result from the DeepSeek model, the final score from Llama model is worse than the one from DeepSeek. However, there might still be room for improvement if we run more iterations, because the evolutionary process is not proven to have converged. 

\begin{figure}[htbp]
    \centering
    \includegraphics[width=\textwidth]{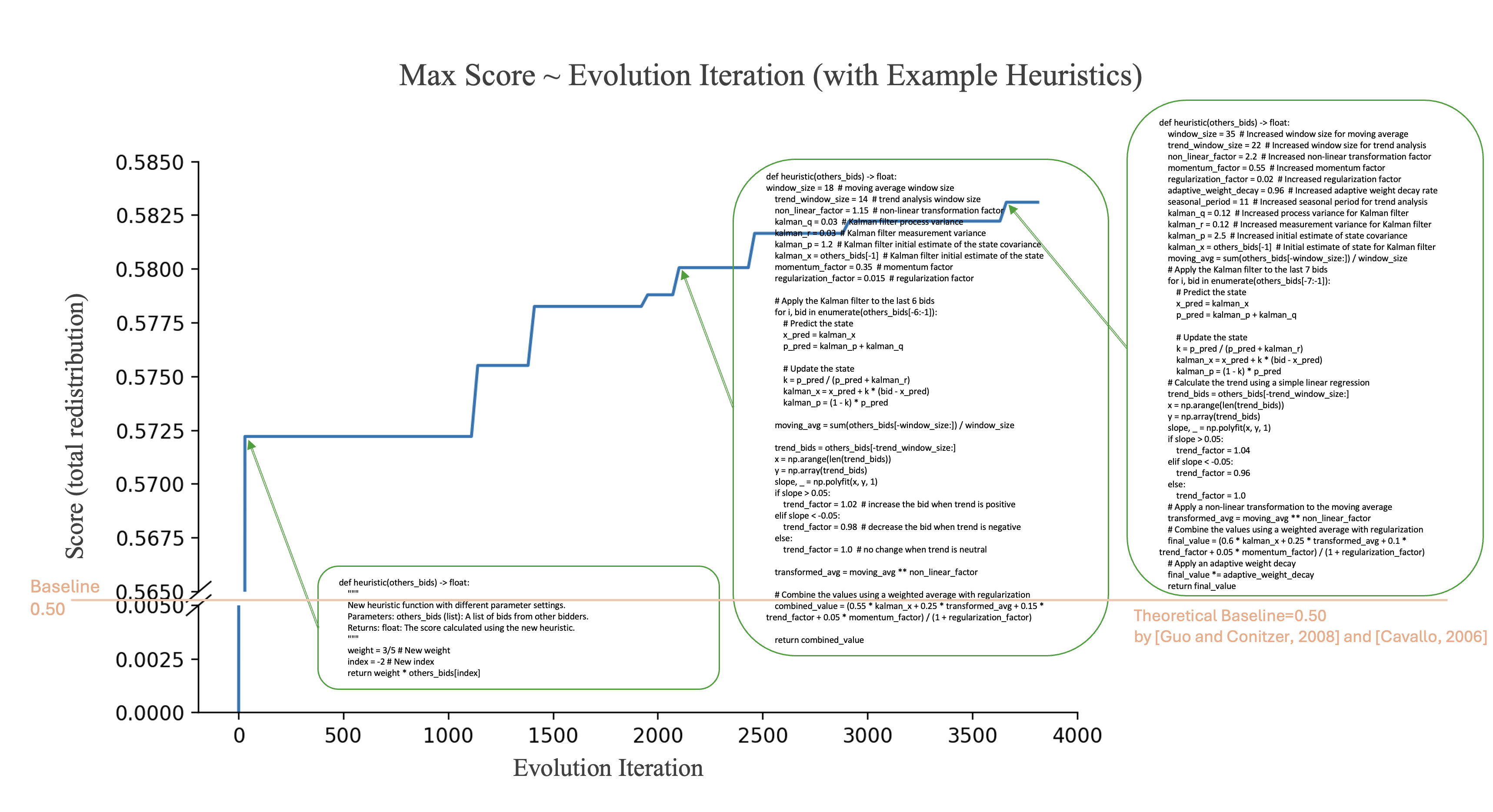}
    \caption{The curve shows the best score so far with the evolution iteration goes on. We can see the best score so far gradually increases. The green boxes contain some example heuristics generated by Llama-3.1-70B-Instruct model, from the corresponding positions in the evolutionary process pointed at by the green arrows. }
    \label{fig:fig_llama_redistri}
\end{figure}

We also did some initial trials using more models but stopped only in a few iterations. 
The results in Fig.~\ref{fig:poor-baseline-funsearch} use the setting of 5 bidders and 2 items, still with the same goal of finding the optimal redistribution mechanism. The evaluation is based on 1,000 random samples from the $U[0,1]$ distribution. 
From this result, we can observe that the model can improve upon existing methods by updating the parameters (from using linspace 1 and 0.5 to using linspace 1 and 0.1 in their code) or come up with new ideas (convolve). 

\begin{figure}[ht!]
    \centering
    \includegraphics[width=\textwidth]{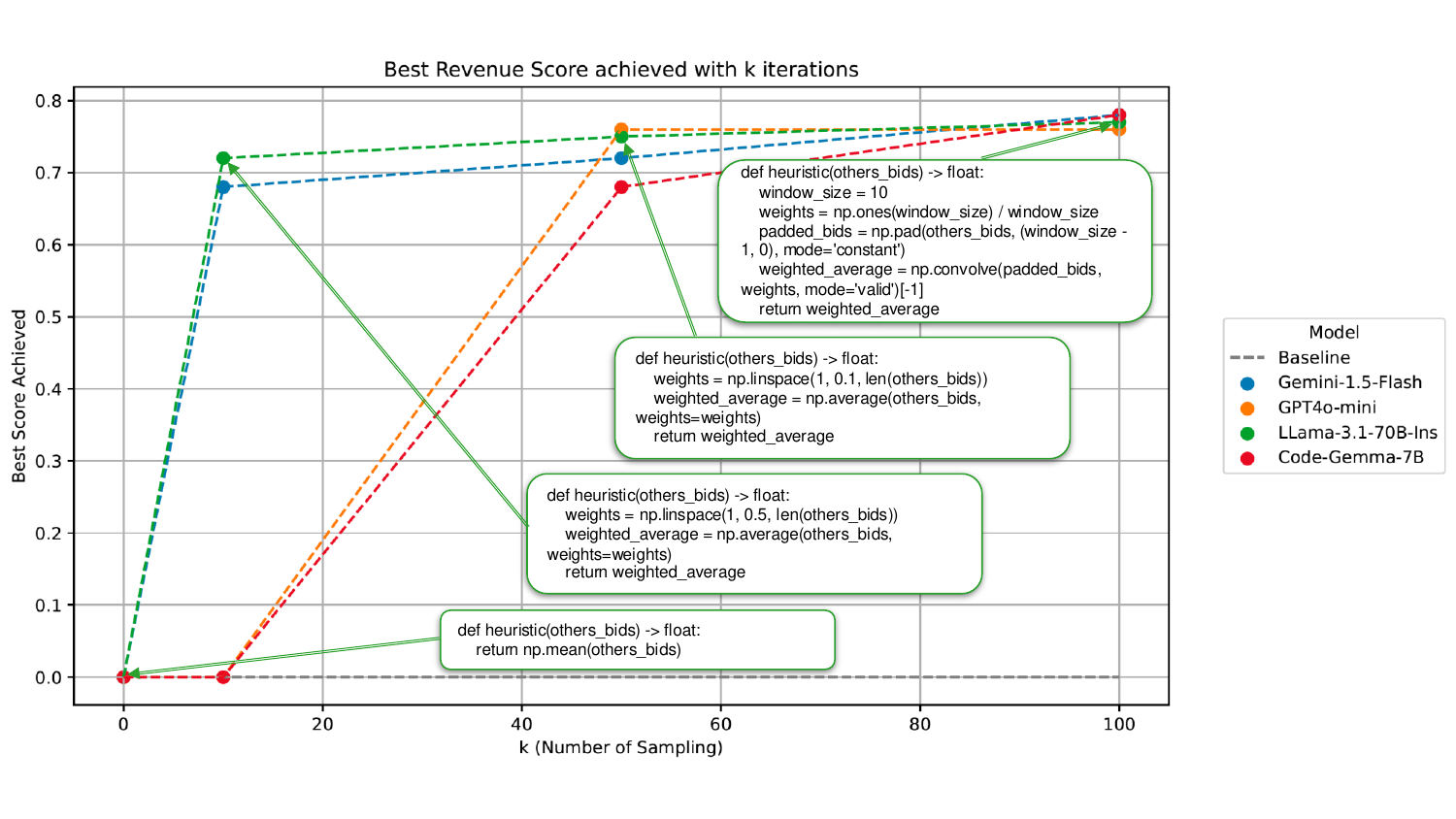}
    \caption{Best Revenue Score achieved with k iterations for several models. The 4 sample programs are heuristic functions output by the LLama-3.1-70B-Ins model from different evolution sample iterations (0, 10, 50, 100, respectively), pointed at by the green arrows. From this result, we can observe that the model can improve upon existing methods by updating the parameters (from using linspace 1 and 0.5 to using linspace 1 and 0.1 in their code), or come up with new ideas (convolve).}
    \label{fig:poor-baseline-funsearch}
\end{figure}

\section{Experiment Details on Designing Correlated Bidder Auction Mechanism}\label{apdx:correlated}

\subsection{Implementation Details}

\subsubsection{Syetem prompt}
The system prompt we use in our framework is as follows, including the problem setting and some requirements for designing the mechanism. 

\begin{tcolorbox}[title= {The System Prompt for Correlated Bidders Auction Design}, colback=gray!10, colframe=gray!50, coltitle=black]\label{prompt:correlated}  
\begin{Verbatim}[breaklines=true,fontsize=\scriptsize,breaksymbol={}]
You are a mechanism design expert. Your task is to design a heuristic function for a deterministic revenue-maximizing single-item auction with 1 item and 2 bidders where the two bidders' values are correlated, i.e., the seller aims to sell the item to one of the two correlated bidders. The heuristic function's input is bidders' bids (a vector of length 2). The heuristic function's output is the allocation function, which is a vector of length 3, each element should be in range [0, 1]; if the first index of the output vector is the largest, then allocate the item to the first bidder; if the second index is the largest, then allocate to the second bidder; if the third index is the largest, then do not allocate to any bidder. The goal of the mechanism designer is to maximize the revenue while ensuring that participants have incentives to bid truthfully (incentive compatibility) and are not worse off by participating (individual rationality). The input to the heuristic function is only the bidders' bids, which is an array of length 2; there are no other args as input, do not use other parameters without defining. The output must be a vector of length 3. If you want to use any new parameters in addition to others_bids, you must initialize that explicitly in the code of heuristic function. You should remove redundant code (i.e., if a shorter version of code has equivalent function, please use the shorter one). You only need to design one new heuristic function. Note: use 2 spaces as indent for Python code. Only output a standalone heuristic_v{version} function code, do not output anything else. You can use various functions that commonly occurs in game theory and mechanism design literature. You should make the code concise and short. 
\end{Verbatim}
\end{tcolorbox}

For some models, we also include a system prompt saying, ``You must not write annotations that are quoted by triple quotes; you can write annotation starting with \#.'' 
This is because the code snippet output by LLMs will be wrapped by three single quotation marks, like ```[code snippet]'''. 
If there are annotations that are quoted by triple quotation marks, then the parsing process sometimes causes an error. We thus include this additional requirement in the prompt, and it turns out that the LLMs we use in the experiments in this section (Llama-3.1-70B-Instruct and DeepSeek-V3) can easily follow this instruction. 

\subsubsection{Grid Distribution}
The correlated distribution we use is a grid distribution. In such a distribution, we divide the area of $[0,1]\times[0,1]$ into $5 \times 5=25$ smaller squares, each with side length 0.2 (each square is $0.1\times 0.1$). In each sub-square, the probability density is uniform. We represent our grid distribution using the following matrix $M$. 
\[
M =
\begin{bmatrix}  
1.397 & 1.168 & 1.428 & 1.043 & 0.751 \\  
0.694 & 1.085 & 0.969 & \textbf{1.083} & 1.034 \\  
1.403 & 0.641 & 1.424 & 1.225 & 1.328 \\  
1.582 & 0.744 & 1.298 & 0.158 & 0.142 \\  
1.272 & 1.117 & 0.557 & 0.513 & 0.946  
\end{bmatrix}  
\]
When a bid profile $(x,y)$ falls into a sub-square -- for example, suppose it falls into the submatrix $[0.2, 0.4]\times [0.6, 0.8]$, which corresponds to the 1.083 value in the matrix -- then the probability density at point $(x,y)$ in this grid distribution is 1.083 with an additional normalization term that 
ensures that the distribution function over the whole area integrates to 1.

\subsection{Evolution Dynamics}

The evolution dynamics are shown in Fig.~\ref{fig:correlated-dynamics-1} and Fig.~\ref{fig:correlated-dynamics-2}. 
There are also some examples of heuristics selected from the evolutionary process shown in Fig.~\ref{fig:correlated-dynamics-1}. We can see that the AMD-LLM framework is able to evolve to find completely different ideas or to explore by modifying some parameters from previously generated heuristics. 

\begin{figure}[htbp]
    \centering
    \includegraphics[width=\textwidth]{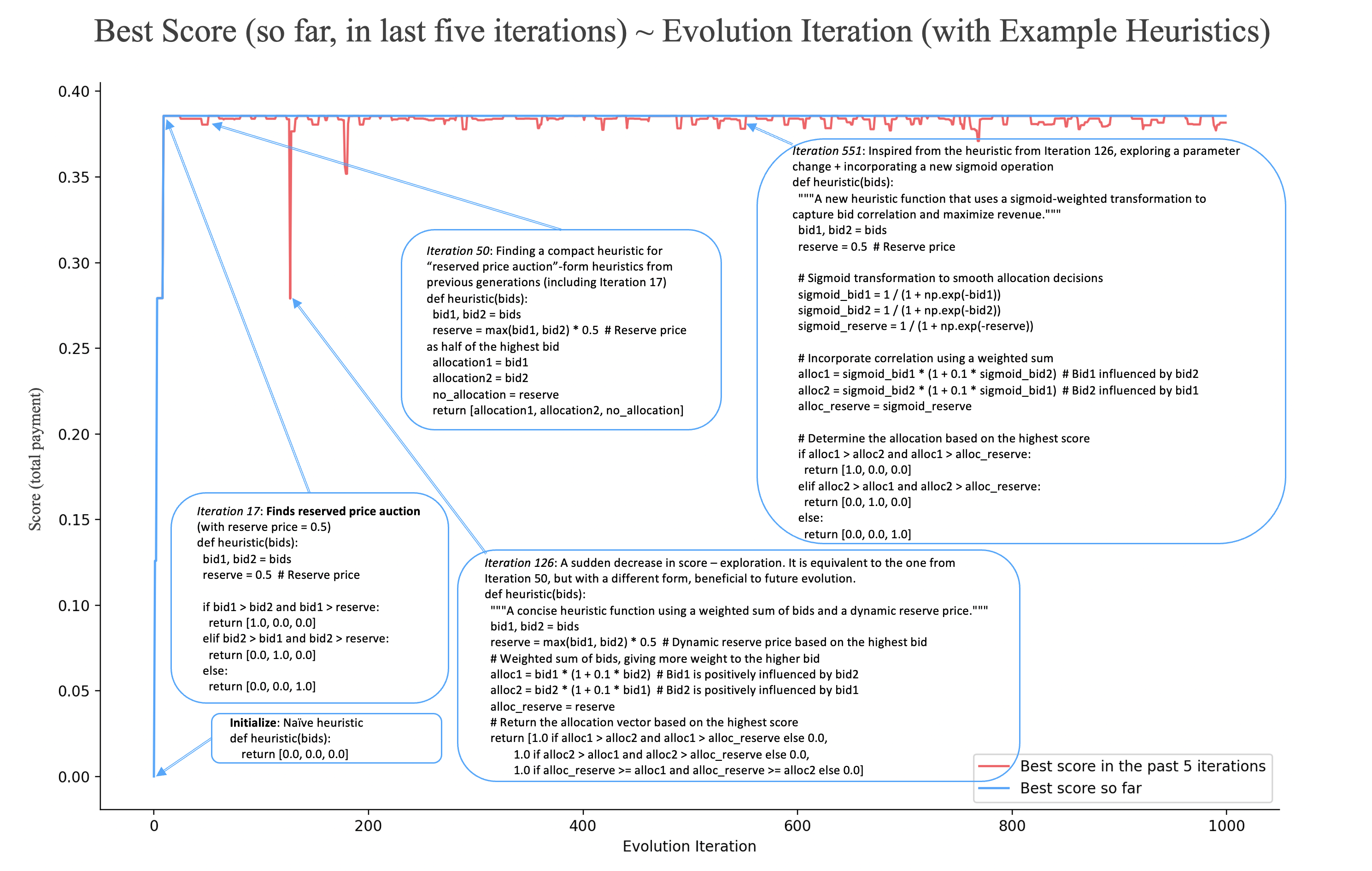}
    \caption{The best average score (so far, in last five iterations) by evolution iteration, using AMD-LLM, on auction design with correlated bidders. The score is the average total payment. The blue boxes contain some example heuristics generated, from the corresponding positions in the evolutionary process, pointed at by the blue arrows.}
    \label{fig:correlated-dynamics-1}
\end{figure}

\begin{figure}[htbp]
    \centering
    \includegraphics[width=0.75\textwidth]{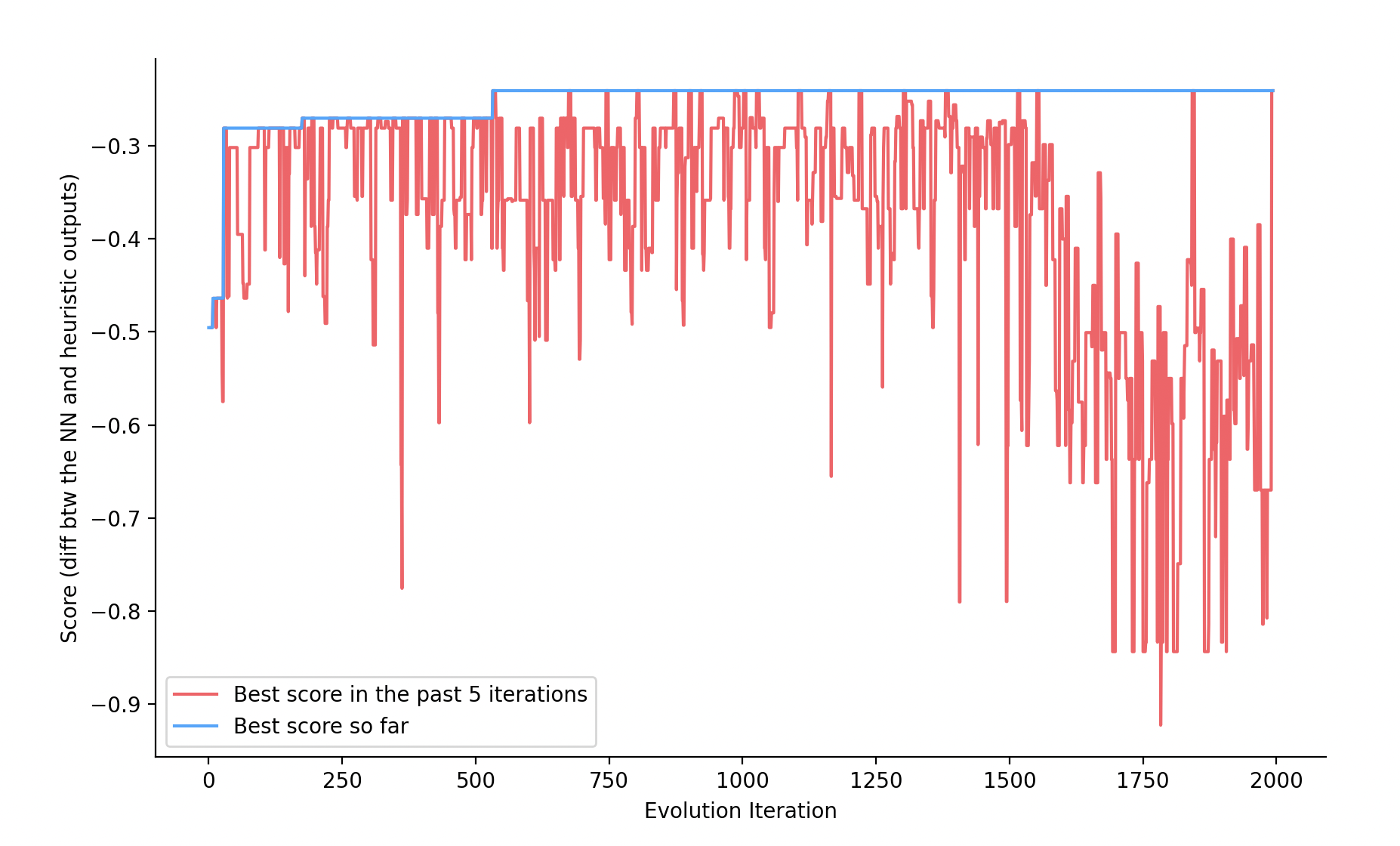}
    \caption{The best average score (so far, in last five iterations) by evolution iteration, using NNA-AMD-LLM, on auction design with correlated bidders. The score is the negated average L1 difference between the heuristic's output and the NN output. }
    \label{fig:correlated-dynamics-2}
\end{figure}

\subsection{More Experimental Results}

Other LLM-generated heuristic examples include dynamic reserve price mechanisms. An example heuristic is as follows. 

\vspace{0.4cm}
\begin{lstlisting}[mathescape=true]
def heuristic(bids):
  """A new heuristic function that uses a multiplicative correlation factor."""
  bid1, bid2 = bids
  correlation_factor = 1.0 - abs(bid1 - bid2)  # Higher correlation if bids are close
  
  # Allocation scores based on bids and correlation
  alloc1 = bid1 * correlation_factor
  alloc2 = bid2 * correlation_factor
  alloc_reserve = 0.5 * (1 - correlation_factor)  # Reserve allocation increases with low correlation
  
  # Normalize the allocation scores to sum to 1
  total = alloc1 + alloc2 + alloc_reserve
  alloc1 /= total
  alloc2 /= total
  alloc_reserve /= total
  
  return [alloc1, alloc2, alloc_reserve]
\end{lstlisting}
\vspace{0.4cm}

In this heuristic, there is a correlation\_factor, which is larger when the two bids are close, so this value actually evaluates the similarity between bids. The reserve allocation increases when bids diverge, creating a non-linear penalty for dissimilar bids. It favors reserve allocation to avoid poor sales when the two bids diverge (it appears that there is less competition in this case because the larger bid dominates). On the other hand, this mechanism allocates more aggressively to high bidders when their bids cluster.

These generated heuristics can offer valuable insights to mechanism designers, helping the development of better manual mechanisms. Furthermore, they can guide the design of improved pre-set evolution strategies for our framework that integrate these insights, potentially enhancing performance and accelerating convergence.

\end{document}